%% file: main.tex
\documentclass[]{academic_template}

\usepackage{microtype}
\usepackage{graphicx}
\usepackage{wrapfig}
\usepackage{needspace}
\usepackage{subcaption}
\usepackage{booktabs}
\usepackage{hyperref}
\usepackage{amsfonts}
\usepackage{amssymb}
\usepackage{amsmath}
\usepackage{amsthm}
\usepackage{mathtools}
\usepackage{multirow}
\usepackage{makecell}
\usepackage{tabularx}
\usepackage{algorithm}
\usepackage{algorithmic}
\usepackage{chngcntr}
\usepackage{float}


\providecommand{\eg}{e.g.}

\newtheorem{theorem}{Theorem}

\theoremstyle{remark}

\title{Enhancing Pretrained Model-based Continual Representation Learning via Guided Random Projection}

\author[1,2,*]{Ruilin Li}
\author[4,*]{Heming Zou}
\author[6]{Xiufeng Yan}
\author[5]{Zheming Liang}
\author[1,2]{Jie Yang}
\author[1,\dagger]{Chenliang Li}
\author[3,\dagger]{\protect\newline Xue Yang}

\affiliation[1]{Wuhan University}
\affiliation[2]{Shanghai Innovation Institute}
\affiliation[3]{Shanghai Jiao Tong University}
\affiliation[4]{Tsinghua University}

\affiliation[5]{University of Science and Technology of China}
\affiliation[6]{China University of Mining and Technology}

\contribution[*]{Equal contribution}
\contribution[\dagger]{Corresponding author}

\abstract{Recent paradigms in Random Projection Layer (RPL)--based continual representation learning have demonstrated superior performance when building upon a pre-trained model (PTM). These methods insert a randomly initialized RPL after a PTM to enhance feature representation in the initial stage. Subsequently, a linear classification head is used for analytic updates in the continual learning stage. However, under severe domain gaps between pre-trained representations and target domains, a randomly initialized RPL exhibits limited expressivity under large domain shifts. While largely scaling up the RPL dimension can improve expressivity, it also induces an ill-conditioned feature matrix, thereby destabilizing the recursive analytic updates of the linear head. To this end, we propose the Stochastic Continual Learner with MemoryGuard Supervisory Mechanism (SCL-MGSM). Unlike random initialization, MGSM constructs the projection layer via a principled, data-guided mechanism that progressively selects target-aligned random bases to adapt the PTM representation to downstream tasks. This facilitates the construction of a compact yet expressive RPL while improving the numerical stability of analytic updates. Extensive experiments on multiple exemplar-free Class Incremental Learning (CIL) benchmarks demonstrate that SCL-MGSM achieves superior performance compared to state-of-the-art methods. 

\textbf{Project Page:} \url{https://rlinl.github.io/SCL/}}

\begin{document}
\maketitle
\input{main_body}
\end{document}

%% file: main_body.tex
\section{Introduction}
Although pre-trained models (PTMs) constitute strong representational backbones, their continual adaptation to downstream tasks is severely impeded by catastrophic forgetting~\cite{mccloskey1989catastrophic,french1999catastrophic,kirkpatrick2017overcoming, li2025etcon}.
This challenge is particularly pronounced in exemplar-free CIL, which operates without task IDs or historical samples~\cite{zhu2021prototype}.
Recently, Random Projection Layer (RPL)-based methods~\cite{zhuang2022acil, li2024harnessing, mcdonnell2024ranpac, li2023crnet, momenianacp, peng2024loranpac, zou2025fly, zou2025structural, zou2025flylora} have emerged as a promising paradigm, particularly for exemplar-free settings.
Typically, the PTM is adapted on an initial task and then frozen. For each new task, data are projected through the frozen PTM and RPL, and only the linear head is updated via recursive least squares and its variants, which are algebraically equivalent to joint ridge regression on all observed data~\cite{greville1960some}.

The efficacy of this paradigm is theoretically grounded in the geometry of high-dimensional spaces.
As formalized by~\citep{peng2023ideal}, scaling up the RPL dimension expands the available null space, thereby providing the necessary geometric degrees of freedom to identify the intersection of task-specific solution spaces and prevent forgetting.
Driven by this ``wider-is-better'' theoretical insight, existing methods often resort to inflating the RPL to extremely high dimensions (e.g., $>$10k random bases) to ensure sufficient separability and mitigate the domain gap~\cite{peng2024loranpac}.

\begin{wrapfigure}{r}{0.58\textwidth}
    \vspace{-0.1em}
    \centering
    \includegraphics[width=\linewidth]{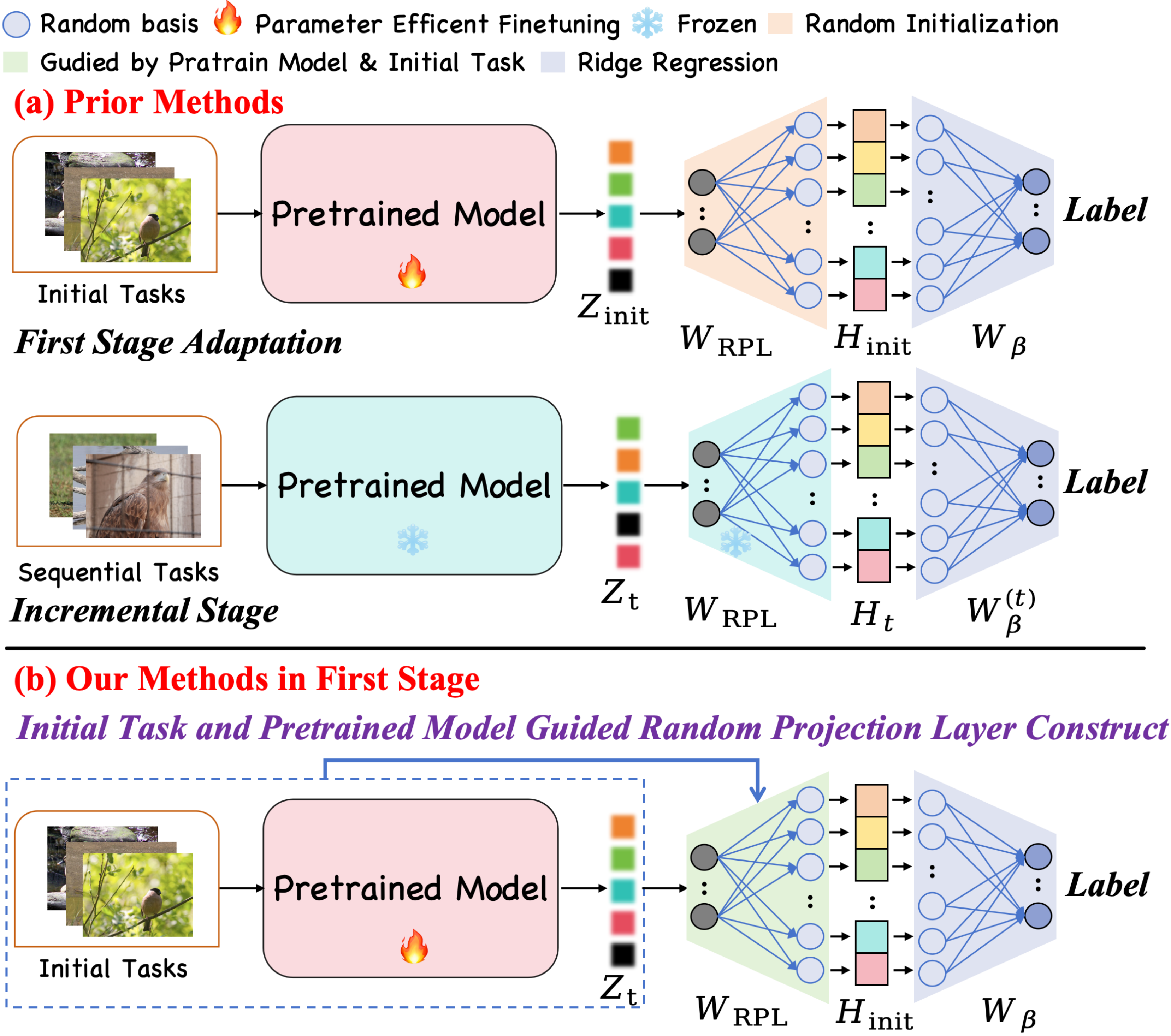}
    \caption{\textbf{Overview of the prior RPL-based CIL paradigm and comparison of initial-stage RPL construction.} \textbf{(a) Prior Methods:} After first stage adaptation, frozen PTM extracts features $\boldsymbol{Z}_{\text{init}}$, which are projected through a randomly initialized RPL ($\boldsymbol{W}_{\text{RPL}}$) to obtain high-dimensional features $\boldsymbol{H}_{\text{init}}$, followed by computing classifier weights $\boldsymbol{W}_{\beta}$. During incremental learning, new features $\boldsymbol{Z}_t$ pass through the same frozen $\boldsymbol{W}_{\text{RPL}}$ to $\boldsymbol{H}_t$, and only $\boldsymbol{W}_{\beta}$ is updated to $\boldsymbol{W}_{\beta}^{(t)}$ via recursive ridge regression. \textbf{(b) Our Method:} We leverage the initial task and PTM to inform MGSM-guided RPL construction. During incremental learning, $\boldsymbol{W}_{\beta}$ is updated recursively as in (a).}
    \label{fig:top}
    
\end{wrapfigure}
However, a critical gap exists between this theoretical ideal and its practical realization. Under a severe domain gap between the pre-trained representations and target domains, unguided random bases are unlikely to cover task-relevant regions in the feature space, causing the RPL to lack the expressivity required for downstream tasks. 
While aggressively scaling up the RPL dimension theoretically mitigates this issue by improving linear separability and expanding the null space to accommodate continual learning, practically it yields a highly ill-conditioned random-feature matrix~\cite{peng2023ideal}. This ill-conditioning forces the ridge-regression solver to rely on additional mechanisms to maintain numerical stability~\cite{mcdonnell2024ranpac,panos2023first}.
Stabilizing analytic updates for a classification head with extremely large dimensions incurs substantial computational overhead.
We posit that such stability and expressivity should arise from the intrinsic quality of the RPL.
Currently, a principled way to configure the RPL that maintains high expressivity in a low dimension while preserving the numerical stability required by analytic classifiers is still lacking.

To address this dilemma, we draw inspiration from first-session adaptation (FSA)~\cite{panos2023first}, where the PTM is fine-tuned on the initial task via Parameter-Efficient Fine-Tuning (PEFT)~\cite{chen2022adaptformer,lian2022scaling,jia2022visual} and then frozen for subsequent tasks to narrow the domain gap. FSA substantially improves performance and has been widely adopted in both RPL-based~\cite{panos2023first,mcdonnell2024ranpac,momenianacp} and prototype-based methods~\cite{zhou2024expandable,zhou2024revisiting}.
Motivated by this insight, we go beyond adapting only the PTM and use the initial task to also guide the construction of the RPL.
\textbf{We therefore propose the Stochastic Continual Learner with MemoryGuard Supervisory Mechanism (SCL-MGSM).}
\textbf{In the initial stage}, MGSM constructs the RPL via a principled, data-guided mechanism: candidate random bases are sampled from an adaptively updated distribution and progressively selected by a target-aligned residual criterion. This produces a compact yet expressive RPL whose dimension is adaptively determined rather than fixed a priori. The convergence of this construction process is theoretically guaranteed. \textbf{In continual learning stages}, this RPL remains frozen. The compact RPL, with less collinear bases, yields better-conditioned feature matrices, improving the numerical stability required by recursive ridge-regression updates.
An overview is provided in Figure \ref{fig:top}.
The core contributions of this work are as follows:

\noindent\textbf{1.}
We propose to go beyond adapting only the PTM and leverage the initial task to also guide the construction of the RPL. This enables an expressive RPL, without resorting to extremely high-dimensional projections.

\noindent\textbf{2.}
We introduce the MemoryGuard Supervisory Mechanism (MGSM), a principled, data-guided mechanism that employs a target-aligned residual criterion to progressively select informative and non-redundant random bases. With theoretical convergence analysis (Theorem~\ref{sm}), MGSM constructs a compact, well-conditioned random feature space that supports stable recursive ridge-regression updates without external compensations.

\noindent\textbf{3.} Extensive experiments on multiple exemplar-free CIL benchmarks demonstrate that SCL-MGSM achieves superior performance and efficiency.

\section{Related Works}

\noindent\textbf{Continual representation learning with Pre-trained Models.}
Adapting PTMs to learn new classes sequentially without forgetting previous ones is a significant challenge.
Recent works primarily follow several strategies:
(1) replay-based methods, which store historical exemplars and interleave them with current-task training to mitigate forgetting~\cite{yan2021dynamically,meng2025diffclass}; 
(2) PEFT-based methods, which freeze the majority of PTM parameters and introduce lightweight trainable modules for task-specific adaptation~\cite{wang2022learning,smith2023coda,qiao2023prompt,sun2024mos,yu2024boosting,chen2025achieving}; and
(3) prototype-based methods, which incrementally maintain class prototypes or subspace statistics for efficient classifier updates~\cite{zhou2024expandable,zhou2024revisiting}. 

\noindent\textbf{Random Projection Layer based Analytic Continual Learning.}
Recent RPL-based analytic continual learning has demonstrated strong performance when building upon PTMs.
ACIL~\citep{zhuang2022acil} and G-ACIL~\cite{zhuang2024g} insert a randomly initialized RPL between the PTM and the classifier to enhance feature representation, and recursively update the classifier via analytic learning.
RanPAC~\cite{mcdonnell2024ranpac} adopts the same RPL strategy and further introduces first-session adaptation (FSA), where the PTM is fine-tuned on the initial task and then frozen for subsequent tasks, to adapt the PTM representations to downstream tasks. This significantly improves performance and is also effective for prototype-based methods.
LoRanPAC~\cite{peng2024loranpac} further scales up the RPL dimension in the first stage to enhance expressivity, and incorporates an additional mechanism to stabilize recursive updates for subsequent tasks, further improving performance.
AnaCP~\cite{momenianacp} introduces a contrastive projection layer to adapt features for each task, also yielding notable improvements.
Our work follows the RPL-based paradigm. Beyond fine-tuning the PTM on the initial task, we use initial-task data to guide RPL construction. This improves RPL expressivity without resorting to extremely high-dimensional projections, yielding better performance while preserving computational efficiency.

\section{Revisiting RPL-based Analytic Continual Learning}
\label{sec:preliminary}

\subsection{Exemplar-free Class-Incremental Learning Setting}
In the CIL setting, the model receives a sequence of datasets $\mathcal{D}_t = \{(\boldsymbol{x}_i, y_i)\}_{i=1}^{N_t}$ for tasks $t = 1, 2, \ldots, T$, where $\boldsymbol{x}_i \in \mathbb{R}^{c \times w \times h}$ denotes an input sample and $y_i \in \mathcal{Y}_t$ is its label. Class sets are disjoint across tasks, i.e., $\mathcal{Y}_t \cap \mathcal{Y}_{\hat{t}} = \varnothing$ for $t \neq \hat{t}$, and task identifiers are unavailable during both training and inference. In the exemplar-free setting, no data from previous stages are stored, so only the current-task dataset $\mathcal{D}_t$ is accessible at stage $t$.

\subsection{The Supervisory Mechanism of SCNs}
We briefly review the Stochastic Configuration Supervisory Mechanism (SCSM) of~\cite{wang2017stochastic}, from which our proposed MGSM departs.
Consider a single-hidden-layer network with $L-1$ hidden units, $f_{L-1}=\boldsymbol{H}_{L-1}\boldsymbol{W}_{\beta_{L-1}}$, where $\boldsymbol{H}_{L-1}=[\boldsymbol{h}_1,\ldots,\boldsymbol{h}_{L-1}]\in\mathbb{R}^{N\times(L-1)}$ is the hidden output matrix with each random basis $\boldsymbol{h}_{i}=g(\boldsymbol{X}\boldsymbol{w}_{i}+b_{i})\in\mathbb{R}^{N}$, and $\boldsymbol{W}_{\beta_{L-1}}\in\mathbb{R}^{(L-1)\times m}$ are the output weights. Here $\boldsymbol{X}\in\mathbb{R}^{N\times d}$ is the input data, $g(\cdot)$ is an activation function, and $\boldsymbol{w}_i$, $b_i$ are randomly sampled input weight and bias. For multi-output residuals $\boldsymbol{e}_{L-1}=[\boldsymbol{e}_{L-1,1},\ldots,\boldsymbol{e}_{L-1,m}]$, a new random basis $\boldsymbol{h}_L$ is accepted only if each output component satisfies:
\begin{equation}
\label{eq:scn_criterion}
    \langle \boldsymbol{e}_{L-1,q}, \boldsymbol{h}_{L} \rangle^{2} \geq\| \boldsymbol{h}_{L} \|^{2} \delta_{L,q},\quad q=1,\ldots,m,
\end{equation}
where $\delta_{L,q}=( 1-r-\mu_{L} ) \| \boldsymbol{e}_{L-1,q} \|^{2}$, with contraction rate $r\in(0,1)$ and vanishing tolerance $0 \le \mu_L \le 1-r$, $\mu_L \to 0$. Under mild regularity conditions, this greedy line-projection criterion guarantees $\lim_{L \to\infty} \| \boldsymbol{e}_{L} \|=0$~\cite{wang2017stochastic}.
In our CIL setting, this criterion often yields compact random features at initialization, but the task-specific selection can bias bases toward current-task residuals and harm cross-stage robustness. To address this limitation, we introduce a novel supervisory mechanism in Section~\ref{sec:method}.

\subsection{Prior RPL-based Continual Representation Learning Methods Framework}
Recent RPL-based methods~\cite{zhuang2022acil,mcdonnell2024ranpac} adopt a three-stage pipeline for exemplar-free CIL.

\textbf{(1) Feature Extraction:} Given the first-stage dataset $\mathcal{D}_{1} = \{(\boldsymbol{x}_i, y_i)\}_{i=1}^{N_1}$, denoted as $\mathcal{D}_{\text{init}}$ (with $N:=N_1$), input samples are fed into the PTM (frozen after first-session adaptation) to obtain feature representations:
\begin{equation}
\label{eq:PTM_prelim}
\boldsymbol{Z}_{\text{init}} = \phi(\boldsymbol{X}_{\text{init}}; \Theta) \in \mathbb{R}^{N \times d},
\end{equation}
where $\boldsymbol{X}_{\text{init}} = [\boldsymbol{x}_1, \ldots, \boldsymbol{x}_N]^\top$, $\phi(\cdot;\Theta)$ denotes the frozen PTM mapping, and $d$ is the embedding dimension.

\textbf{(2) RPL Initialization:} A Random Projection Layer is constructed by randomly sampling a weight matrix $\boldsymbol{W} \in \mathbb{R}^{d \times L}$ and a bias vector $\boldsymbol{b} \in \mathbb{R}^{L}$ from a predefined distribution, yielding the random feature matrix $\boldsymbol{H}_{\mathrm{init}} = g(\boldsymbol{Z}_\text{init}\boldsymbol{W} + \boldsymbol{1}_N\boldsymbol{b}^\top) \in \mathbb{R}^{N \times L}$. Driven by the ``wider-is-better'' theoretical insight~\cite{cover1965geometrical,peng2023ideal}, recent methods often inflate $L \gg N$ to ensure sufficient linear separability and expand the null space for accommodating continual learning.

\textbf{(3) Classifier Analytic Update:} Given initial-task random features $\boldsymbol{H}_{\mathrm{init}}$ and label matrix $\boldsymbol{Y}_{\mathrm{init}}=[\boldsymbol{y}_1,\ldots,\boldsymbol{y}_N]^\top\in\{0,1\}^{N\times C_{\mathrm{init}}}$, where $\boldsymbol{y}_i$ is the one-hot encoding of $y_i$ and $C_{\mathrm{init}}=|\mathcal{Y}_1|$ is the number of initial-task classes, the classifier is first obtained by ridge-regularized least squares:
\begin{equation}
\boldsymbol{W}_{\beta}^{(1)} = \arg\min_{\boldsymbol{W}_{\beta}}
\left\|\boldsymbol{H}_{\mathrm{init}}\boldsymbol{W}_{\beta}-\boldsymbol{Y}_{\mathrm{init}}\right\|_F^2
+ \lambda\left\|\boldsymbol{W}_{\beta}\right\|_F^2.
\end{equation}
The corresponding sufficient statistic is
\begin{equation}
\boldsymbol{P}_{\mathrm{init}} = \boldsymbol{H}_{\mathrm{init}}^{\top}\boldsymbol{H}_{\mathrm{init}} + \lambda \boldsymbol{I}. \label{eq:P_init}
\end{equation}
For subsequent tasks $t = 2, 3, \dots$, recursive least squares (RLS) updates~\cite{greville1960some} are:
\begin{align}
\boldsymbol{P}_{t} &= \boldsymbol{P}_{t-1} + \boldsymbol{H}_t^\top \boldsymbol{H}_t, \label{eq:P_update_simple}\\
\boldsymbol{W}_{\beta}^{(t)} &= \boldsymbol{W}_{\beta}^{(t-1)} + \boldsymbol{P}_t^{-1} \boldsymbol{H}_t^\top
(\boldsymbol{Y}_t - \boldsymbol{H}_t \boldsymbol{W}_{\beta}^{(t-1)}), \label{eq:W_update_simple}
\end{align}
which is algebraically equivalent to joint ridge regression on all observed data, requiring only current-task data and the sufficient statistics $(\boldsymbol{W}_{\beta}^{(t-1)}, \boldsymbol{P}_{t-1})$. Method-specific variants may modify the update form but follow the same recursive update principle.
In practice, this recursion is numerically reliable when $\boldsymbol{P}_{t}$ remains well-conditioned.
If the RPL dimension is very large, unguided random bases can become highly collinear, which may increase the condition number of $\boldsymbol{H}_{k}^{\top}\boldsymbol{H}_{k}$ and degrade the conditioning of $\boldsymbol{P}_{t}$ when regularization is insufficient relative to the spectral spread.
Under finite-precision arithmetic, the inverse in Eq.~\eqref{eq:W_update_simple} may then amplify numerical errors and destabilize incremental learning.
\textbf{This suggests a practical trade-off in RPL-based continual learning: unguidedly increasing the number of random bases can improve representational expressivity, but it can also reduce the numerical robustness of recursive updates.}

\section{Method}
\label{sec:method}

\begin{figure*}[t]
		\begin{center}
			\includegraphics[width=\linewidth]{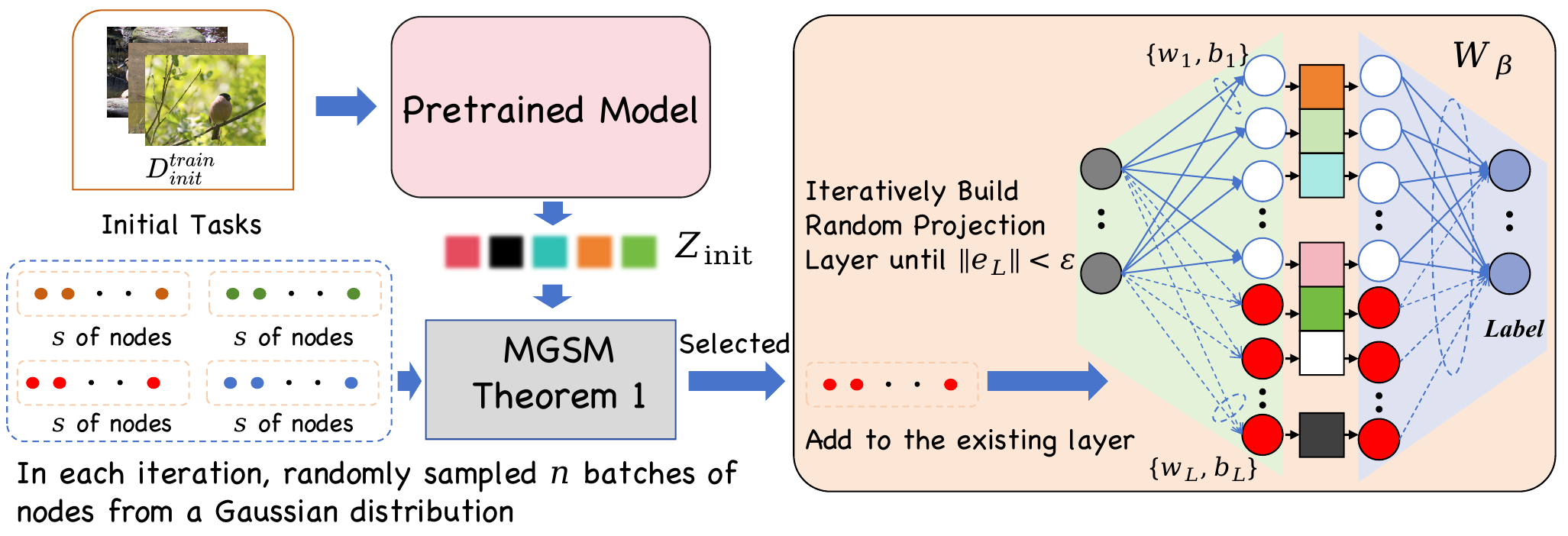}
		\caption{\textbf{Overview of MGSM-guided RPL construction in SCL-MGSM.} Data from any stage can serve as the initialization set to build the RPL from scratch. Random hidden units are progressively sampled, evaluated by MGSM, and appended to the RPL only if they satisfy the supervisory criterion. The construction terminates once the residual converges below a predefined threshold $\varepsilon$. See Appendix~\ref{a_MGSM_process} for details.}
	\label{fig:pipline}
		\end{center}
	
\end{figure*}

We propose SCL-MGSM, whose core is the MemoryGuard Supervisory Mechanism (MGSM), a principled data-guided framework for progressively constructing an expressive yet well-conditioned RPL.
Motivated by first-session adaptation (FSA)~\cite{panos2023first}, MGSM leverages initial-task data and the pretrained model to guide RPL construction, iteratively selecting target-aligned, non-redundant random bases rather than sampling them all at once.
The overall framework is illustrated in Fig.~\ref{fig:pipline}.
In Section~\ref{sec:MGSM}, we detail the algorithmic workflow of MGSM step by step (see Algorithm~\ref{Alg_MGSM} for the complete procedure).
In Section~\ref{sec:rationale}, we explain how MGSM yields an expressive RPL whose well-conditioned feature matrix ensures numerically stable recursive analytic updates.

\subsection{SCL-MGSM Construction}
\label{sec:MGSM}

\noindent \textbf{Random Sampling of Hidden Units.}
Let $\boldsymbol{Z}_t$ denote the stage-$t$ backbone features obtained via Eq.~\eqref{eq:PTM_prelim}, with $\boldsymbol{Z}_{\text{init}}$ as the initialization-stage instance.
Each candidate hidden unit is defined as
\begin{equation}
\boldsymbol{h}_i
= g\!\left(\boldsymbol{Z}_{\mathrm{init}}\boldsymbol{w}_i + b_i\right)
\in \mathbb{R}^{N \times 1},
\label{eq:basic-block}
\end{equation}
where $\boldsymbol{w}_i$ and $b_i$ are the input weight and bias of the $i$-th hidden unit,
both randomly sampled from $\mathcal{N}(0,\xi^2)$ with scaling factor $\xi$.
When a candidate is accepted by MGSM, its activation vector $\boldsymbol{h}_i$
is appended as a column of the RPL matrix.

\noindent \textbf{Incremental Construction of the RPL.} SCL-MGSM incrementally constructs the RPL via an iterative forward configuration.
After $(L-s)$ hidden units have been accepted, their activations on
$\boldsymbol{Z}_{\mathrm{init}}$ form the current RPL output
\begin{equation}
\boldsymbol{H}_{L-s}
= [\boldsymbol{h}_1,\ldots,\boldsymbol{h}_{L-s}]
\in \mathbb{R}^{N \times (L-s)}.
\end{equation}
The corresponding classifier is
\begin{equation}
f_{L-s}(\boldsymbol{Z}_{\mathrm{init}})
= \boldsymbol{H}_{L-s}\boldsymbol{W}_{\beta_{L-s}},
\end{equation}
where $\boldsymbol{W}_{\beta_{L-s}} \in \mathbb{R}^{(L-s)\times C}$ are the readout
weights mapping the RPL output to the $C$ class scores.

To decide whether further expansion is required, we define the current
multi-output residual (with label matrix $\boldsymbol{Y}\in\mathbb{R}^{N\times C}$)
\begin{equation}
    \boldsymbol{E}_{L-s}
    =
    \boldsymbol{Y}
    -
    \boldsymbol{H}_{L-s}\boldsymbol{W}_{\beta_{L-s}},
\end{equation}
and monitor its Frobenius norm $\|\boldsymbol{E}_{L-s}\|_{F}$.
If $\|\boldsymbol{E}_{L-s}\|_{F}\le\varepsilon$ for a predefined tolerance
$\varepsilon>0$, the RPL construction completes.

Otherwise, we continue to expand the RPL: at each iteration, we draw $B_{\max}$ candidate blocks.
For each $j\in\{1,\ldots,B_{\max}\}$, we independently sample $s$ hidden units
with parameters drawn from $\mathcal{N}(0,\xi^2)$ and compute their activations
on $\boldsymbol{Z}_{\mathrm{init}}$ to obtain a candidate matrix $\boldsymbol{H}_{s}^{(j)} \in \mathbb{R}^{N\times s}$.
Using the augmented hidden matrix
$[\boldsymbol{H}_{L-s},\boldsymbol{H}_{s}^{(j)}]$, we evaluate the $j$-th
candidate block by the MGSM acceptance criterion (formalized in
Theorem~\ref{sm}).

Among all candidates that satisfy the criterion, we select the one achieving the
largest improvement and accept it, so that the output matrix is updated to
\begin{equation}
    \boldsymbol{H}_{L}
    =
    [\boldsymbol{H}_{L-s},\boldsymbol{H}_{s}^{(j^\ast)}].
\end{equation}

\noindent\textbf{MGSM Acceptance Criterion.}
We now formalize the block-wise selection rule used above.
Let $\boldsymbol{H}_s \in \mathbb{R}^{N\times s}$ be the output matrix of a batch of $s$ new hidden units, and define the augmented hidden matrix
$\boldsymbol{H}_L = [\boldsymbol{H}_{L-s}, \boldsymbol{H}_s]$. Set
\begin{equation}
\begin{aligned}
    \boldsymbol{S} &= \left( \boldsymbol{H}_{s}^{\top} \boldsymbol{H}_{s} + \lambda \boldsymbol{I}_{s} \right) 
    - \boldsymbol{H}_{s}^{\top} \boldsymbol{H}_{L-s}
    \left( \boldsymbol{H}^{\top}_{L-s} \boldsymbol{H}_{L-s} + \lambda \boldsymbol{I}_{L-s} \right)^{-1}
    \boldsymbol{H}^{\top}_{L-s} \boldsymbol{H}_{s}.
\end{aligned}
\end{equation}
For theoretical conciseness, we next state the criterion for a single-output
target column. Let
$\boldsymbol{e}_{L-s}=\boldsymbol{y}-\boldsymbol{H}_{L-s}\boldsymbol{W}_{\beta_{L-s}}$,
and define
\begin{equation}
\label{eq:th_def1}
    \boldsymbol{v} \;=\; \boldsymbol{H}_s^{\top}\,\boldsymbol{e}_{L-s}.
\end{equation}

\begin{theorem}
  \label{sm}
  Let $\boldsymbol{y}\in\mathbb{R}^{N}$ be the target vector, and let $\boldsymbol{H}_{L-s} \in \mathbb{R}^{N \times (L-s)}$ be the output matrix of the current network $f_{L-s}$. Suppose $
  \boldsymbol{W}_{\beta_{L-s}}  = [\beta_1,\dots,\beta_{L-s}]^{\top}$ is the output weights, and define the current residual:
  \begin{equation}
      \boldsymbol{e}_{L-s} \;=\; \boldsymbol{y} - \boldsymbol{H}_{L-s}\,\boldsymbol{W}_{\beta_{L-s}}.
  \end{equation}

  If the batch of new hidden units with output $\boldsymbol{H}_s$ satisfy the following inequality and $\mathcal{R}_L \geq 0$ (Eq.~\ref{eq:coupling_term}):
  \begin{equation}
  \label{eq:core_sm}
      2 \boldsymbol{v}^{\top} \, \boldsymbol{S}^{-1} \, \boldsymbol{v}-\boldsymbol{v}^{\top} \Bigl( \boldsymbol{S}^{-1} \boldsymbol{H}_{s}^{\top} \boldsymbol{H}_{s} \, \boldsymbol{S}^{-1} \Bigr) \, \boldsymbol{v} \geq( 1-r ) \, \| \boldsymbol{e}_{L-s} \|^{2},
  \end{equation}
  where $0<r<1$ and the output weights $\boldsymbol{W}_{\beta_{L}^{\star }} = [\beta^{\star }_1,\dots,\beta^{\star }_{L}]^{\top}$ are the ridge-regression weights that minimize

\begin{equation}
\label{update_weight}
\boldsymbol{W}_{\beta_{L}}^{\star}
= \operatorname*{arg\,min}_{\boldsymbol{W}\in\mathbb{R}^{L}}
\;\bigl\|\boldsymbol{y}-\boldsymbol{H}_{L}\boldsymbol{W}\bigr\|_{2}^{2}
+ \lambda \bigl\|\boldsymbol{W}\bigr\|_{2}^{2}.
\end{equation}

      then $\lim_{L \to\infty} \|\boldsymbol{e}_{L} \|=0$, where $\boldsymbol{e}_{L} \;=\; \boldsymbol{y} - \boldsymbol{H}_{L}\boldsymbol{W}_{\beta_{L}^{\star }}$.
  \end{theorem}
\begin{proof}
See Appendix \ref{proof of MGSM}.
\end{proof}

\noindent\textit{Remark (multi-output extension).}
For classification with $\boldsymbol{Y} \in \mathbb{R}^{N\times C}$,
Theorem~\ref{sm} is applied independently to each output column, which is
equivalent to the Frobenius-norm formulation used in our construction.

If no candidate block is accepted under the current scaling $\xi$, we switch to
the next value $\xi \in\mathcal{X}_{\xi}=\{\xi_{\min},\xi_{\min}+\Delta\xi,\ldots,\xi_{\max}\}$,
resample $B_{\max}$ candidate blocks from $\mathcal{N}(0,{\xi}^{2})$, and
re-evaluate them by the same criterion.
Empirically, exploring this discrete scaling set reliably yields admissible
blocks, broadens the effective function space, and improves robustness to
diverse task distributions, while keeping the PTM frozen and relying only on
forward computations.

\noindent \textbf{Initial SCL-MGSM Classifier.} Let $L^\ast$ denote the final number of hidden units in the constructed RPL.
The SCL-MGSM classifier at the initial stage is
\begin{equation}
\label{eq:head}
f_{L^*}(\boldsymbol{Z}_{\mathrm{init}}) = \boldsymbol{H}_{L^*} \boldsymbol{W}_{\beta} = \sum_{i=1}^{L^*} \boldsymbol{h}_i \boldsymbol{\beta}_i^{\top} \in \mathbb{R}^{N \times C},
\end{equation}
where $\boldsymbol{W}_\beta$ is obtained by ridge-regularized least squares.
This procedure yields a task-relevant RPL characterized by the output matrix
$\boldsymbol{H}_{L^\ast}$.

\noindent\textbf{Incremental Update of Output Weights.}
The recursive update rule follows the generic RPL-based CIL formulation in Section~\ref{sec:preliminary}, i.e., Eqs.~\eqref{eq:P_update_simple}--\eqref{eq:W_update_simple}. In our method, $\boldsymbol{H}_{\mathrm{init}}=\boldsymbol{H}_{L^\ast}$, and thus Eq.~\eqref{eq:P_init} initializes the sufficient statistic for subsequent exemplar-free updates.

\subsection{Underlying Rationale}
\label{sec:rationale}
\begin{figure}[t]
    \centering
    \begin{minipage}[t]{0.31\linewidth}
        \centering
        \includegraphics[height=2.5cm,keepaspectratio]{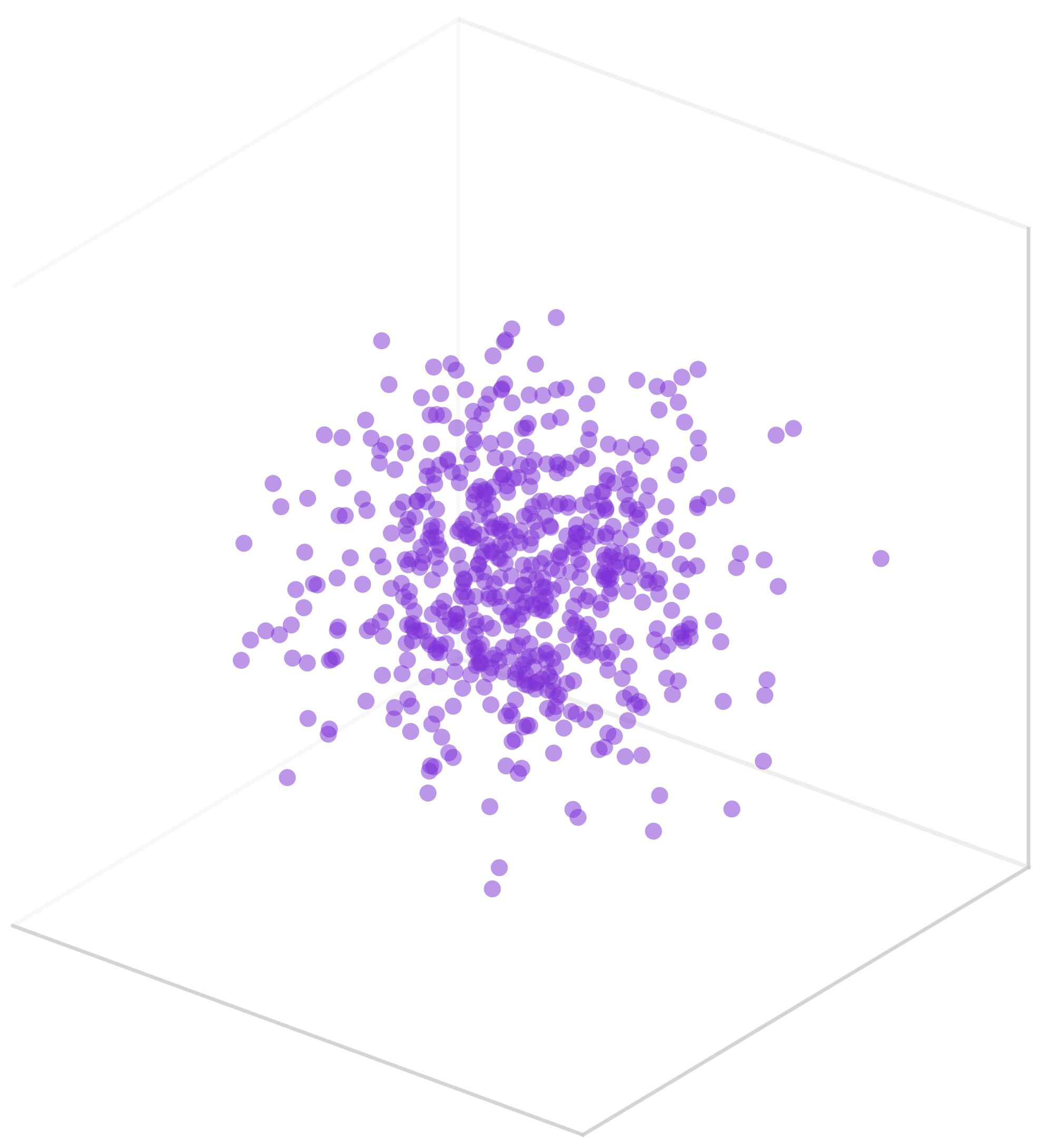}
        \vspace{\smallskipamount}
        \begin{minipage}[t][2cm][t]{\linewidth}
            \captionof{figure}{\textbf{Gaussian Initialization.}}
            \label{fig:block1}
        \end{minipage}
    \end{minipage}
    \hfill
    \begin{minipage}[t]{0.31\linewidth}
        \centering
        \includegraphics[height=2.5cm,keepaspectratio]{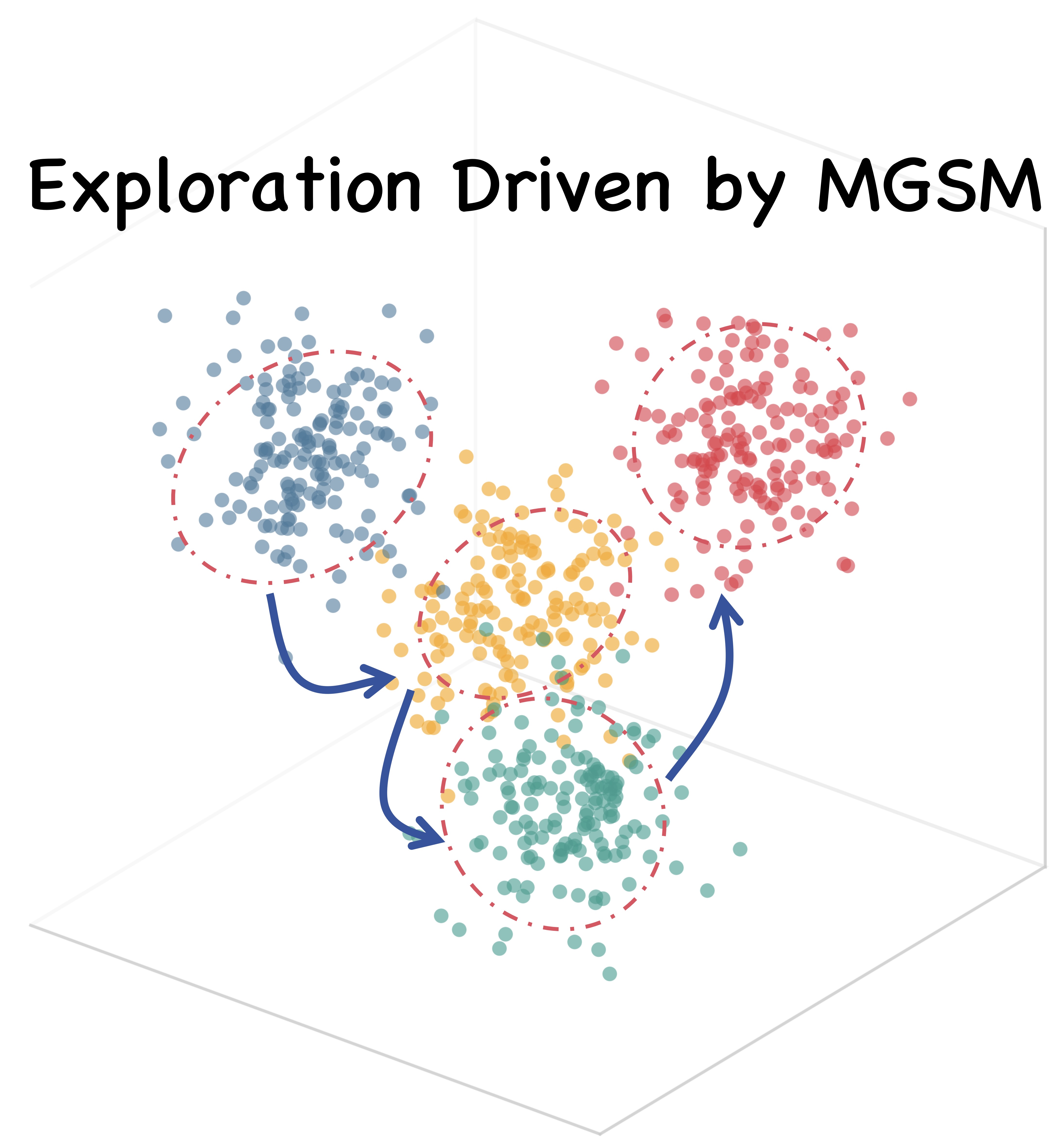}
        \vspace{\smallskipamount}
        \begin{minipage}[t][2cm][t]{\linewidth}
            \captionof{figure}{\textbf{Visualization of MGSM Exploration Strategy.}}
            \label{fig:block2}
        \end{minipage}
    \end{minipage}
    \hfill
    \begin{minipage}[t]{0.31\linewidth}
        \centering
        \includegraphics[height=2.5cm,keepaspectratio]{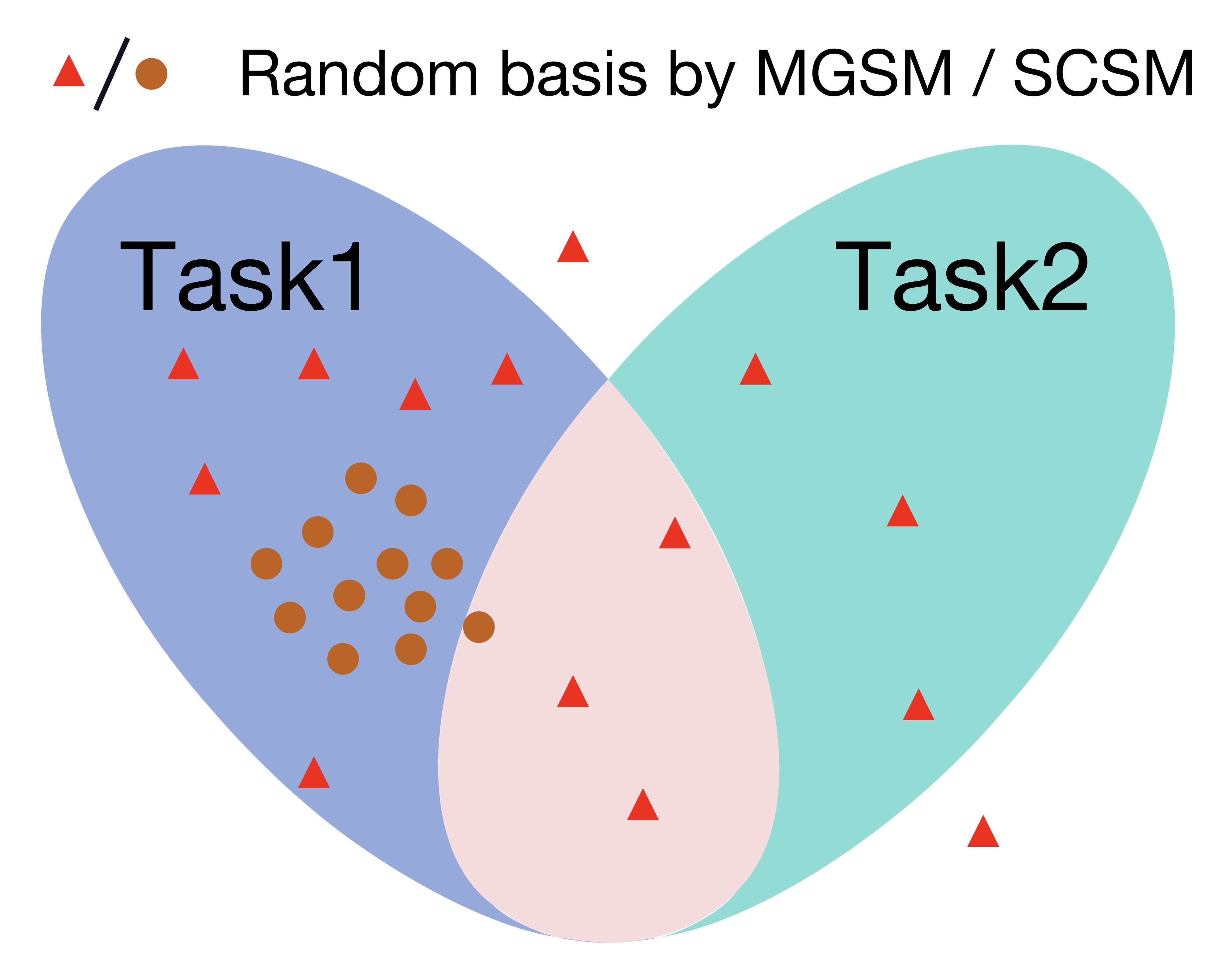}
        \vspace{\smallskipamount}
        \begin{minipage}[t][2cm][t]{\linewidth}
            \captionof{figure}{\textbf{Illustration of Random-Basis and Task-Relevant Regions.}}
            \label{fig:block3}
        \end{minipage}
    \end{minipage}
 \vskip -0.5in
\end{figure}

We now explain how MGSM simultaneously enhances RPL expressivity and preserves the numerical stability required by recursive analytic updates.

\noindent\textbf{MGSM-guided RPL construction enhances expressivity.}
First-session adaptation (FSA) adapts PTM representations using initial-task data~\cite{panos2023first,mcdonnell2024ranpac,zhou2024expandable}, which presumably provides a useful inductive bias for subsequent tasks that share statistical similarities.
Our work goes one step further by using the initial-task data to guide RPL construction. SCL-MGSM incrementally configures the RPL via a block-wise, data-driven procedure: candidate random bases are sampled from an adaptively updated parameter distribution and accepted only if the target-aligned residual criterion is met. Otherwise, the sampling distribution is adjusted and the search is redirected to explore a different scale (Fig.~\ref{fig:block2}).
This multi-scale exploration steers the search toward parameter regions better suited for downstream continual learning tasks, producing a more functionally diverse set of random bases that enhances the expressivity of the RPL (Fig.~\ref{fig:block3}). Theorem~\ref{sm} guarantees the convergence of this construction process.
In comparison, fixed-distribution sampling concentrates random bases on a narrow shell in parameter space (Fig.~\ref{fig:block1}), so conventional random initialization can only improve expressivity by inflating the RPL dimension.

\noindent\textbf{MGSM enables numerically stable analytic updates.}
MGSM-guided RPL achieves high expressivity without resorting to extremely high-dimensional projections, thereby avoiding the ill-conditioning that compromises analytic updates.
Furthermore, compared to the supervisory mechanism of SCSM~\cite{wang2017stochastic}, which also aims to achieve expressivity in low dimensions, MGSM adopts a relaxed, target-aligned acceptance criterion instead of SCSM's greedy residual-aligned criterion (Fig.~\ref{fig:block3}).
Combined with ridge regularization and block-wise updates, this relaxed criterion admits basis combinations that are not individually optimal but jointly effective, producing less collinear bases and better-conditioned Gram matrices for stable incremental updates.

\begin{algorithm}[h]
	\caption{SCL-MGSM}
	\label{Alg_MGSM}
    \textbf{Input}: Initial Task $\mathcal{D}_1$, Incremental Tasks $\{\mathcal{D}_t\}_{t=2}^{T}$, predefined error $\varepsilon$, max batches $B_{max}$, number of nodes $s$.
    \begin{algorithmic}[1] 
    \STATE \textit{(Optional) First-session adaptation (FSA) of the PTM.}
		\STATE \textbf{\textit{\# Modeling RPL During Initial Training}}
        \STATE Extract features from frozen PTMs by Eq.~\eqref{eq:PTM_prelim}
		\WHILE{The predefined residual error is not satisfied}
		    \STATE Recruit randomly generated nodes based on MGSM by Eq. (\ref{eq:core_sm})
		    \STATE Update output weights by Eq. (\ref{update_weight})
		\ENDWHILE
		\STATE Compute irreversible intermediate matrix $\boldsymbol{P}_{\mathrm{init}}$ by Eq. (\ref{eq:P_init})
		\STATE \textbf{\textit{\# Update Output Weights During Incremental Learning.}}
            \FOR{$t=2, 3,\dots,T$}
            \STATE Extract features from frozen PTMs by Eq.~\eqref{eq:PTM_prelim}
		    \STATE Update Output Weights by Eq. (\ref{eq:W_update_simple})
            \STATE Update irreversible intermediate matrix by Eq. (\ref{eq:P_update_simple})
            \ENDFOR
\end{algorithmic}
	\textbf{Output}: SCL-MGSM
\end{algorithm}

\section{Experiments}
\label{Experiments}

\subsection{Experiment Setting}

\textbf{Datasets:}
We conduct exemplar-free CIL experiments on four datasets, including ImageNet-R~\cite{hendrycks2021many}, ImageNet-A~\cite{hendrycks2021natural}, ObjectNet~\cite{barbu2019objectnet}, and Omnibenchmark~\cite{zhang2022benchmarking}. ImageNet-R, ImageNet-A, and ObjectNet each contain 200 classes, while Omnibenchmark contains 300 classes. For brevity, we refer to them as IN-R, IN-A, ObjNet, and OmniB, respectively.

\noindent\textbf{Compared Methods:}
We compare our proposed SCL-MGSM with four categories of CIL methods. (1) RPL-based methods: AnaCP~\cite{momenianacp}, LoRanPAC~\cite{peng2024loranpac}, G-ACIL~\cite{zhuang2024g}, RanPAC~\cite{mcdonnell2024ranpac}, KLDA~\cite{momeni2025continual}. (2) Prototype-based methods: EASE~\cite{zhou2024expandable}, SimpleCIL~\cite{zhou2024revisiting}. (3) Prompt-tuning-based methods: CODA-Prompt~\cite{smith2023coda}, APT~\cite{chen2025achieving}. (4) LoRA-based methods: SD-LoRA~\cite{wu2025sd}.

\noindent\textbf{Evaluation Protocol.}
We report two metrics for all methods: average accuracy \(A_{\text{avg}}=\frac{1}{t+1}\sum_{i=1}^{t}A_i\), where \(A_i\) is the test accuracy on \(D_{1:i}^{\text{test}}\) after the \(i\)-th incremental session, and final accuracy \(A_{\text{last}}\), which is the test accuracy on all learned tasks after the last incremental session.

\noindent\textbf{Implementation Details.}
For fair comparison, all methods are implemented using the frozen ViT-B/16-IN21K~\cite{rw2019timm} as PTMs.
For first-session adaptation and ablation studies, we use
AdaptFormer~\cite{chen2022adaptformer}, SSF~\cite{lian2022scaling}, and VPT~\cite{jia2022visual} as PEFT methods.
All results in Table \ref{tab:all_results} are averaged over 3 random seeds with standard error, where the initial-stage categories are non-overlapping across seeds. Detailed hyperparameter settings can be found in Appendix~\ref{a_details}.
\subsection{Main Results}

\begin{table*}[!t]

\caption{\textbf{Comparison of CIL Performance Across Methods and Benchmarks.} \textbf{Bold} and \underline{underline} mark the best and second-best results. $A_{\text{last}}$ and $A_{\text{avg}}$ denote final and average accuracy (\%). B-$m$ Inc-$n$ ($m$: initial classes, $n$: incremental classes; \textit{$m{=}0$}: equal split). Joint linear probe trains only the classifier head on all tasks, while Joint fine-tuning updates the entire PTM on all tasks.}
\vskip -0.05in
  \centering
  \setlength{\tabcolsep}{2pt}
  \scriptsize
    \resizebox{\textwidth}{!}{
    \begin{tabular}{l*{8}{c}}
      \toprule
      \textbf{Method} &
      \multicolumn{4}{c}{\textbf{\makecell{ImageNet-R  (B-0)}}} &
      \multicolumn{4}{c}{\textbf{\makecell{ImageNet-A  (B-0)}}} \\
      \cmidrule(lr){2-5} \cmidrule(lr){6-9}
      Joint linear probe &
        \multicolumn{4}{c}{67.23$_{\scriptscriptstyle\pm0.05}$} &
        \multicolumn{4}{c}{50.98$_{\scriptscriptstyle\pm0.65}$} \\
      Joint fine-tuning &
        \multicolumn{4}{c}{83.67$_{\scriptscriptstyle\pm0.14}$} &
        \multicolumn{4}{c}{60.99$_{\scriptscriptstyle\pm0.32}$} \\
      \cmidrule(lr){1-9}
      &
      \multicolumn{2}{c}{\textbf{Inc-5}} &
      \multicolumn{2}{c}{\textbf{Inc-10}} &
      \multicolumn{2}{c}{\textbf{Inc-5}} &
      \multicolumn{2}{c}{\textbf{Inc-10}} \\
      \cmidrule(lr){2-3} \cmidrule(lr){4-5} \cmidrule(lr){6-7} \cmidrule(lr){8-9}
      &
      \textbf{$A_{\text{last}}$} & \textbf{$A_{\text{avg}}$} &
      \textbf{$A_{\text{last}}$} & \textbf{$A_{\text{avg}}$} &
      \textbf{$A_{\text{last}}$} & \textbf{$A_{\text{avg}}$} &
      \textbf{$A_{\text{last}}$} & \textbf{$A_{\text{avg}}$} \\
      \midrule
      CODA-Prompt &
        57.52$_{\scriptscriptstyle\pm0.68}$ & 65.71$_{\scriptscriptstyle\pm0.00}$ & 68.03$_{\scriptscriptstyle\pm0.42}$ & 74.57$_{\scriptscriptstyle\pm0.74}$
        & 28.81$_{\scriptscriptstyle\pm0.60}$ & 41.45$_{\scriptscriptstyle\pm0.47}$ & 39.21$_{\scriptscriptstyle\pm0.83}$ & 50.17$_{\scriptscriptstyle\pm0.97}$ \\
      SD-LoRA &
        66.29$_{\scriptscriptstyle\pm1.10}$ & 73.67$_{\scriptscriptstyle\pm1.19}$ & 72.25$_{\scriptscriptstyle\pm0.72}$ & 78.14$_{\scriptscriptstyle\pm0.80}$
        & 32.92$_{\scriptscriptstyle\pm1.07}$ & 46.00$_{\scriptscriptstyle\pm1.32}$ & 45.95$_{\scriptscriptstyle\pm0.53}$ & 58.41$_{\scriptscriptstyle\pm1.01}$ \\
      SimpleCIL &
        54.55$_{\scriptscriptstyle\pm0.00}$ & 61.63$_{\scriptscriptstyle\pm0.79}$ & 54.55$_{\scriptscriptstyle\pm0.00}$ & 61.22$_{\scriptscriptstyle\pm0.60}$
        & 48.85$_{\scriptscriptstyle\pm0.00}$ & 60.49$_{\scriptscriptstyle\pm0.91}$ & 48.85$_{\scriptscriptstyle\pm0.00}$ & 59.96$_{\scriptscriptstyle\pm0.88}$ \\
      EASE &
        70.65$_{\scriptscriptstyle\pm0.52}$ & \underline{76.69$_{\scriptscriptstyle\pm0.24}$} & 74.13$_{\scriptscriptstyle\pm0.57}$ & \underline{80.38$_{\scriptscriptstyle\pm0.60}$}
        & 43.27$_{\scriptscriptstyle\pm2.55}$ & 55.80$_{\scriptscriptstyle\pm0.46}$ & 45.97$_{\scriptscriptstyle\pm1.93}$ & 58.09$_{\scriptscriptstyle\pm2.42}$ \\

      APT &
        70.55$_{\scriptscriptstyle\pm0.40}$ & 76.46$_{\scriptscriptstyle\pm0.71}$ & \underline{75.07$_{\scriptscriptstyle\pm0.32}$} & 79.97$_{\scriptscriptstyle\pm0.57}$
        & 37.19$_{\scriptscriptstyle\pm2.11}$ & 45.72$_{\scriptscriptstyle\pm3.80}$ & 51.57$_{\scriptscriptstyle\pm0.99}$ & 61.72$_{\scriptscriptstyle\pm0.44}$ \\
      KLDA &
        64.82$_{\scriptscriptstyle\pm0.30}$ & 71.02$_{\scriptscriptstyle\pm0.66}$ & 64.82$_{\scriptscriptstyle\pm0.30}$ & 70.69$_{\scriptscriptstyle\pm0.50}$
        & 51.64$_{\scriptscriptstyle\pm0.16}$ & 58.23$_{\scriptscriptstyle\pm0.76}$ & 51.64$_{\scriptscriptstyle\pm0.16}$ & 57.69$_{\scriptscriptstyle\pm0.84}$ \\
      RanPAC &
        68.72$_{\scriptscriptstyle\pm0.40}$ & 74.84$_{\scriptscriptstyle\pm0.35}$ & 71.25$_{\scriptscriptstyle\pm0.45}$ & 76.58$_{\scriptscriptstyle\pm0.47}$
        & 53.50$_{\scriptscriptstyle\pm1.58}$ & 60.56$_{\scriptscriptstyle\pm1.20}$ & 52.67$_{\scriptscriptstyle\pm0.39}$ & 61.59$_{\scriptscriptstyle\pm0.76}$ \\
      G-ACIL &
        67.05$_{\scriptscriptstyle\pm0.30}$ & 73.95$_{\scriptscriptstyle\pm0.44}$ & 67.05$_{\scriptscriptstyle\pm0.30}$ & 73.64$_{\scriptscriptstyle\pm0.29}$
        & 45.05$_{\scriptscriptstyle\pm0.46}$ & 57.78$_{\scriptscriptstyle\pm0.29}$ & 45.05$_{\scriptscriptstyle\pm0.46}$ & 57.26$_{\scriptscriptstyle\pm0.37}$ \\
      LoRanPAC &
        70.78$_{\scriptscriptstyle\pm0.12}$ & 76.58$_{\scriptscriptstyle\pm0.20}$ & 71.81$_{\scriptscriptstyle\pm0.09}$ & 77.12$_{\scriptscriptstyle\pm0.06}$
        & \underline{54.84$_{\scriptscriptstyle\pm0.25}$} & \underline{64.57$_{\scriptscriptstyle\pm0.38}$} & 54.62$_{\scriptscriptstyle\pm0.31}$ & \underline{64.16$_{\scriptscriptstyle\pm0.55}$} \\
      AnaCP &
        \underline{71.34$_{\scriptscriptstyle\pm0.46}$} & 76.53$_{\scriptscriptstyle\pm0.80}$ & 73.60$_{\scriptscriptstyle\pm1.15}$ & 78.85$_{\scriptscriptstyle\pm0.32}$
        & 53.08$_{\scriptscriptstyle\pm0.64}$ & 62.80$_{\scriptscriptstyle\pm0.55}$ & \underline{55.22$_{\scriptscriptstyle\pm0.39}$} & 63.99$_{\scriptscriptstyle\pm0.61}$ \\
      \midrule

      \makecell[l]{SCL-MGSM} &
        \textbf{72.64$_{\scriptscriptstyle\pm0.24}$} & \textbf{77.70$_{\scriptscriptstyle\pm0.14}$} & \textbf{77.12$_{\scriptscriptstyle\pm0.21}$} & \textbf{82.23$_{\scriptscriptstyle\pm0.18}$}
        & \textbf{55.79$_{\scriptscriptstyle\pm0.23}$} & \textbf{64.93$_{\scriptscriptstyle\pm0.31}$} & \textbf{56.05$_{\scriptscriptstyle\pm0.32}$} & \textbf{65.29$_{\scriptscriptstyle\pm0.47}$} \\
      \midrule
      \addlinespace[0.3em]
      \textbf{Method} &
      \multicolumn{4}{c}{\textbf{\makecell{ObjectNet  (B-0)}}} &
      \multicolumn{4}{c}{\textbf{\makecell{OmniBenchmark  (B-0)}}} \\
      \cmidrule(lr){2-5} \cmidrule(lr){6-9}
      Joint linear probe &
        \multicolumn{4}{c}{55.90$_{\scriptscriptstyle\pm0.28}$} &
        \multicolumn{4}{c}{78.56$_{\scriptscriptstyle\pm0.20}$} \\
      Joint fine-tuning &
        \multicolumn{4}{c}{67.53$_{\scriptscriptstyle\pm0.13}$} &
        \multicolumn{4}{c}{82.27 $_{\scriptscriptstyle\pm0.26}$} \\
      \cmidrule(lr){1-9}
      &
      \multicolumn{2}{c}{\textbf{Inc-5}} &
      \multicolumn{2}{c}{\textbf{Inc-10}} &
      \multicolumn{2}{c}{\textbf{Inc-5}} &
      \multicolumn{2}{c}{\textbf{Inc-10}} \\
      \cmidrule(lr){2-3} \cmidrule(lr){4-5} \cmidrule(lr){6-7} \cmidrule(lr){8-9}
      &
      \textbf{$A_{\text{last}}$} & \textbf{$A_{\text{avg}}$} &
      \textbf{$A_{\text{last}}$} & \textbf{$A_{\text{avg}}$} &
      \textbf{$A_{\text{last}}$} & \textbf{$A_{\text{avg}}$} &
      \textbf{$A_{\text{last}}$} & \textbf{$A_{\text{avg}}$} \\
      \midrule
      CODA-Prompt &
        45.17$_{\scriptscriptstyle\pm0.49}$ & 56.44$_{\scriptscriptstyle\pm2.05}$ & 53.14$_{\scriptscriptstyle\pm0.07}$ & 63.68$_{\scriptscriptstyle\pm1.88}$
        & 58.31$_{\scriptscriptstyle\pm0.80}$ & 69.56$_{\scriptscriptstyle\pm0.33}$ & 64.17$_{\scriptscriptstyle\pm0.58}$ & 74.06$_{\scriptscriptstyle\pm0.48}$ \\
      SD-LoRA &
        50.02$_{\scriptscriptstyle\pm1.24}$ & 61.95$_{\scriptscriptstyle\pm0.60}$ & 55.72$_{\scriptscriptstyle\pm0.98}$ & 66.74$_{\scriptscriptstyle\pm1.81}$
        & 55.90$_{\scriptscriptstyle\pm1.55}$ & 69.22$_{\scriptscriptstyle\pm1.10}$ & 64.04$_{\scriptscriptstyle\pm0.15}$ & 74.25$_{\scriptscriptstyle\pm0.92}$ \\
      SimpleCIL &
        53.58$_{\scriptscriptstyle\pm0.00}$ & 63.15$_{\scriptscriptstyle\pm2.06}$ & 53.58$_{\scriptscriptstyle\pm0.00}$ & 62.70$_{\scriptscriptstyle\pm1.97}$
        & 73.15$_{\scriptscriptstyle\pm0.00}$ & 81.10$_{\scriptscriptstyle\pm0.58}$ & 73.15$_{\scriptscriptstyle\pm0.00}$ & 80.85$_{\scriptscriptstyle\pm0.58}$ \\
      EASE &
        54.34$_{\scriptscriptstyle\pm0.71}$ & 65.30$_{\scriptscriptstyle\pm2.67}$ & 57.59$_{\scriptscriptstyle\pm0.33}$ & 68.34$_{\scriptscriptstyle\pm2.41}$
        & 73.02$_{\scriptscriptstyle\pm0.20}$ & 81.09$_{\scriptscriptstyle\pm0.80}$ & 73.27$_{\scriptscriptstyle\pm0.29}$ & 81.04$_{\scriptscriptstyle\pm0.59}$ \\

      APT &
        56.05$_{\scriptscriptstyle\pm0.51}$ & 66.41$_{\scriptscriptstyle\pm1.02}$ & 60.54$_{\scriptscriptstyle\pm0.22}$ & 70.69$_{\scriptscriptstyle\pm1.60}$
        & 62.99$_{\scriptscriptstyle\pm0.20}$ & 73.17$_{\scriptscriptstyle\pm0.23}$ & 66.12$_{\scriptscriptstyle\pm0.39}$ & 74.97$_{\scriptscriptstyle\pm0.04}$ \\
      KLDA &
        57.92$_{\scriptscriptstyle\pm0.04}$ & 67.23$_{\scriptscriptstyle\pm1.89}$ & 57.92$_{\scriptscriptstyle\pm0.04}$ & 66.76$_{\scriptscriptstyle\pm1.83}$
        & 75.27$_{\scriptscriptstyle\pm0.04}$ & 83.48$_{\scriptscriptstyle\pm0.43}$ & 75.27$_{\scriptscriptstyle\pm0.04}$ & 83.23$_{\scriptscriptstyle\pm0.42}$ \\
      RanPAC &
        59.03$_{\scriptscriptstyle\pm0.41}$ & 69.02$_{\scriptscriptstyle\pm1.76}$ & 61.90$_{\scriptscriptstyle\pm0.41}$ & 71.27$_{\scriptscriptstyle\pm1.84}$
        & 77.37$_{\scriptscriptstyle\pm0.23}$ & 85.33$_{\scriptscriptstyle\pm0.59}$ & 77.34$_{\scriptscriptstyle\pm0.17}$ & 85.00$_{\scriptscriptstyle\pm0.57}$ \\
      G-ACIL &
        57.60$_{\scriptscriptstyle\pm0.26}$ & 68.64$_{\scriptscriptstyle\pm2.39}$ & 57.60$_{\scriptscriptstyle\pm0.26}$ & 68.19$_{\scriptscriptstyle\pm2.29}$
        & 76.41$_{\scriptscriptstyle\pm0.19}$ & 84.94$_{\scriptscriptstyle\pm0.75}$ & 76.41$_{\scriptscriptstyle\pm0.17}$ & 84.70$_{\scriptscriptstyle\pm0.73}$ \\
      LoRanPAC &
        61.89$_{\scriptscriptstyle\pm0.20}$ & 71.35$_{\scriptscriptstyle\pm1.93}$ & 63.57$_{\scriptscriptstyle\pm0.30}$ & 72.48$_{\scriptscriptstyle\pm1.82}$
        & 76.68$_{\scriptscriptstyle\pm0.05}$ & 85.04$_{\scriptscriptstyle\pm0.80}$ & 76.70$_{\scriptscriptstyle\pm0.33}$ & 84.96$_{\scriptscriptstyle\pm0.69}$ \\
      AnaCP &
        \underline{62.66$_{\scriptscriptstyle\pm0.62}$} & \underline{72.14$_{\scriptscriptstyle\pm0.15}$} & \underline{64.68$_{\scriptscriptstyle\pm0.17}$} & \underline{72.81$_{\scriptscriptstyle\pm0.94}$}
        & \underline{78.64$_{\scriptscriptstyle\pm0.26}$} & \underline{86.27$_{\scriptscriptstyle\pm1.08}$} & \underline{78.81$_{\scriptscriptstyle\pm0.24}$} & \underline{86.45$_{\scriptscriptstyle\pm0.71}$} \\
        
      \midrule

      \makecell[l]{SCL-MGSM} &
        \textbf{63.87$_{\scriptscriptstyle\pm0.25}$} & \textbf{73.40$_{\scriptscriptstyle\pm0.92}$} & \textbf{66.45$_{\scriptscriptstyle\pm0.45}$} & \textbf{75.30$_{\scriptscriptstyle\pm2.03}$}
        & \textbf{79.83$_{\scriptscriptstyle\pm0.33}$} & \textbf{86.79$_{\scriptscriptstyle\pm0.60}$} & \textbf{79.90$_{\scriptscriptstyle\pm0.35}$} & \textbf{86.92$_{\scriptscriptstyle\pm0.73}$} \\
      \bottomrule
    \end{tabular}
    }

    \label{tab:all_results}
    \vskip -0.1in
\end{table*}

As shown in Table~\ref{tab:all_results}, RPL-based methods consistently outperform prompt-tuning, LoRA-based, and prototype-based methods on most benchmarks, benefiting from random feature-space expansion. 
Among all compared methods, SCL-MGSM achieves the best $A_{\text{last}}$ and $A_{\text{avg}}$ across all eight evaluation settings, demonstrating consistent superiority over existing approaches.
Specifically, on ImageNet-R with 20 sequential tasks, SCL-MGSM surpasses the second-best method by +2.05\% in $A_{\text{last}}$ and +1.85\% in $A_{\text{avg}}$. On ObjectNet with 20 sequential tasks, the margins reach +1.77\% and +2.49\%, respectively. On ImageNet-A, SCL-MGSM consistently outperforms all competitors under both 40-task and 20-task protocols.
Moreover, SCL-MGSM surpasses the joint linear probe on all datasets, and on ObjectNet the gap to joint fine-tuning, which serves as the upper bound, narrows to approximately 1\%.
Notably, the initial-stage categories are non-overlapping across the three seeds, yet SCL-MGSM yields consistent gains in all cases, confirming that MGSM-guided RPL construction generalizes across different initialization tasks.
These gains are primarily attributed to MGSM, a data-guided mechanism in the initial stage that progressively selects target-aligned random bases with convergence guarantees (Theorem~\ref{sm}), constructing a compact yet expressive RPL that enhances the PTM's representations, even without directly training the PTM on each task. The resulting well-conditioned feature space further benefits the stability of recursive analytic updates during incremental stages. The final hidden sizes of SCL-MGSM are reported in Appendix.

\begin{figure*}[!t]
\centering
\begin{minipage}[t]{0.4\textwidth}
    \centering
    \vspace{0pt}
    {\captionsetup{position=top}
    \captionof{table}{\textbf{Performance Comparison of RPL Construction Strategies Without FSA.} }
    \label{tab:RPL_performance}}
    
\resizebox{\linewidth}{!}{
        \begin{tabular}{lccc}
        \toprule
        \textbf{Dataset} & \textbf{RPL} & \textbf{$A_{\mathrm{Last}}$(\%)} & \textbf{$A_{\mathrm{Avg}}$(\%)} \\
        \midrule
        \multirow{3}{*}{\shortstack {\textit{IN-R}}} & MGSM & \textbf{70.09±0.13} & \textbf{75.90±0.08}  \\
        & SCSM  & 64.78±0.09   & 70.51±0.44  \\
        & RI    & 63.00±0.22   & 72.05±0.39  \\
        \midrule
        \multirow{3}{*}{\shortstack {\textit{IN-A}}} & MGSM & \textbf{53.63±0.74} & \textbf{62.45±0.59}  \\
        & SCSM  & 44.85±1.47 & 54.40±1.88  \\
        & RI    & 47.31±0.43 & 58.21±0.37  \\
        \midrule
        \multirow{3}{*}{\shortstack {\textit{ObjNet}}} & MGSM & \textbf{59.61±0.30}  & \textbf{69.12±2.31} \\
        & SCSM  & 55.46±0.33 & 64.98±2.43 \\
        & RI    & 55.91±0.14 & 66.08±1.94  \\
        \midrule
        \multirow{3}{*}{\shortstack {\textit{OmniB}}} & MGSM & \textbf{77.99±0.17}  & \textbf{85.67±0.81} \\
        & SCSM  & 71.43±0.11 & 79.34±0.93 \\
        & RI    & 74.91±0.04 & 82.54±0.68  \\
        \bottomrule
        \end{tabular}
}
\end{minipage}
\hfill
\begin{minipage}[t]{0.58\textwidth}
    \centering
    \vspace{0pt}
    \vspace*{1.2em}
    \includegraphics[width=\linewidth]{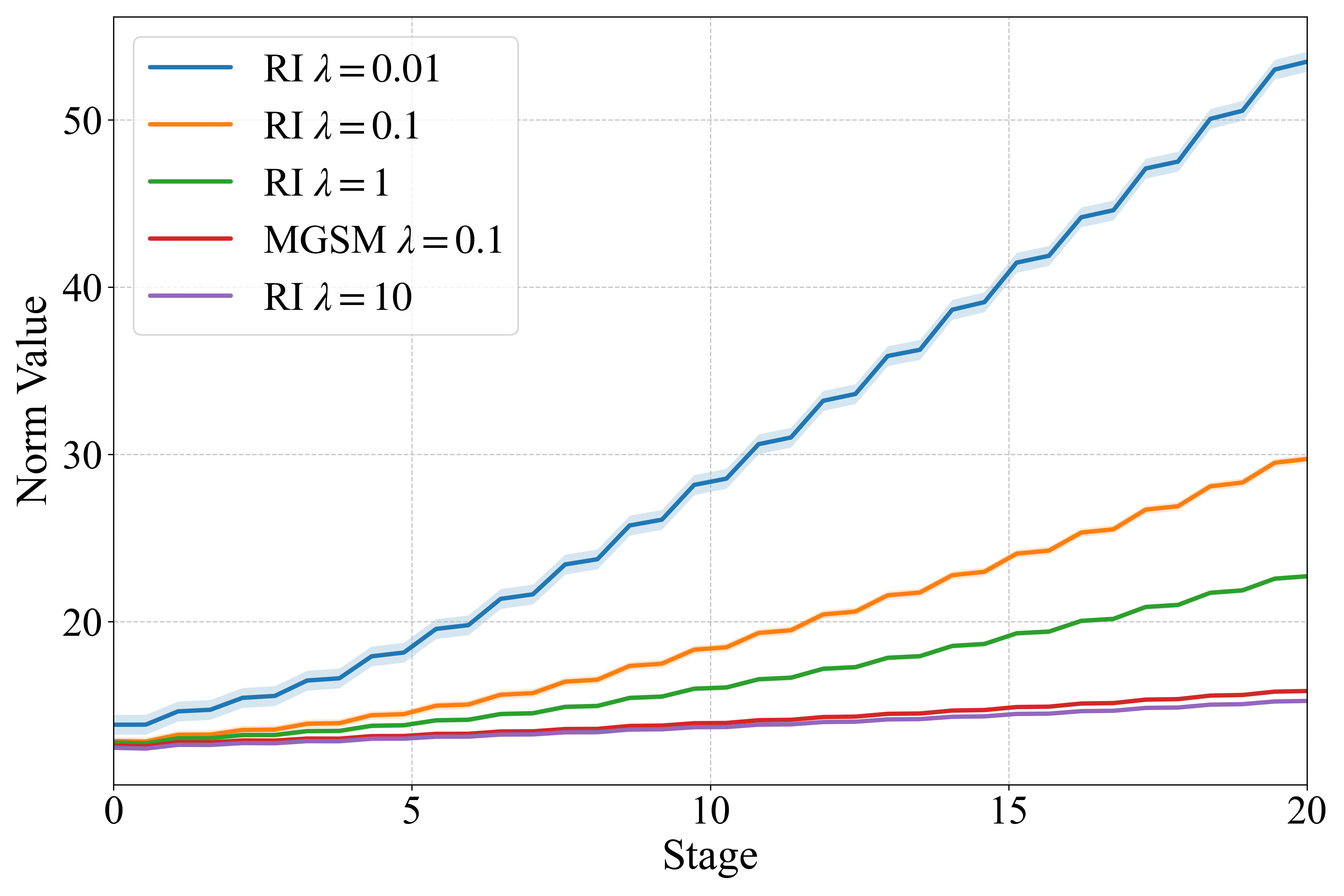}
    \captionof{figure}{\textbf{Comparison of $\boldsymbol{P}_t$ Norm Curves.}}
    \label{fig:n_c}
\end{minipage}

\end{figure*}

\subsection{Mechanism Analysis}

\noindent\textbf{Expressivity under different RPL construction strategies.}
To assess whether MGSM improves RPL expressivity, we compare three construction strategies: MGSM, SCSM (the supervisory mechanism in SCNs~\cite{wang2017stochastic}), and random initialization (RI).
For a fair comparison, all methods yield an RPL with a final hidden size of 10{,}000 without FSA, and the output weights during incremental stages are computed via recursive ridge regression.
As shown in Table~\ref{tab:RPL_performance}, MGSM consistently outperforms both alternatives across all four benchmarks, achieving gains of approximately 5\%--10\% in $A_{\mathrm{Last}}$ over SCSM.
On datasets with severe domain gaps (\eg ImageNet-A), RI degrades substantially, confirming that an RPL composed of limited unguided random bases lacks the expressivity required for out-of-distribution data.
SCSM, while data-guided, employs an overly strict greedy criterion that over-specializes to the initial task, often underperforming even RI on frozen PTM features.
In contrast, MGSM’s relaxed, target-aligned criterion selects diverse, non-redundant bases that enrich RPL expressivity while keeping the projections adapted to downstream tasks.
We further visualize the basis quality on ImageNet-A in Figure~\ref{fig:heatmap}: MGSM produces an almost diagonal cosine-similarity matrix, indicating near-orthogonal basis vectors that enrich the function space, whereas SCSM exhibits pronounced off-diagonal values reflecting high redundancy among bases. This further confirms that MGSM constructs a more expressive RPL, which contributes to stronger CIL performance.

\noindent\textbf{Stability of MGSM-based continual learning.}
As discussed in Section~\ref{sec:preliminary}, stable recursive updates require a well-conditioned random feature matrix.
On ImageNet-A (Figure~\ref{fig:condition_number}), we estimate the condition number of the random feature matrix on a randomly sampled subset of samples. Under MGSM-guided RPL, it is comparable to RI, whereas under SCSM-guided RPL it is markedly larger, indicating a substantially more ill-conditioned matrix and weaker numerical stability for recursive updates.
We also track $\lVert \boldsymbol{P}_t \rVert_F$ over incremental stages (Figure~\ref{fig:n_c}): for RI, $\lVert \boldsymbol{P}_t \rVert_F$ grows steadily, and suppressing this growth requires an excessively large $\lambda$ (\eg $\lambda=10$) that causes severe underfitting. MGSM with a moderate $\lambda=0.1$ maintains a controlled $\lVert \boldsymbol{P}_t \rVert_F$ comparable to heavily regularized RI.
These results confirm that MGSM's criterion contributes to both expressivity and stability.

\begin{figure*}[!t]
    \centering
    \begin{minipage}[t]{0.44\textwidth}
        \centering
        \includegraphics[height=0.16\textheight,width=\linewidth,keepaspectratio]{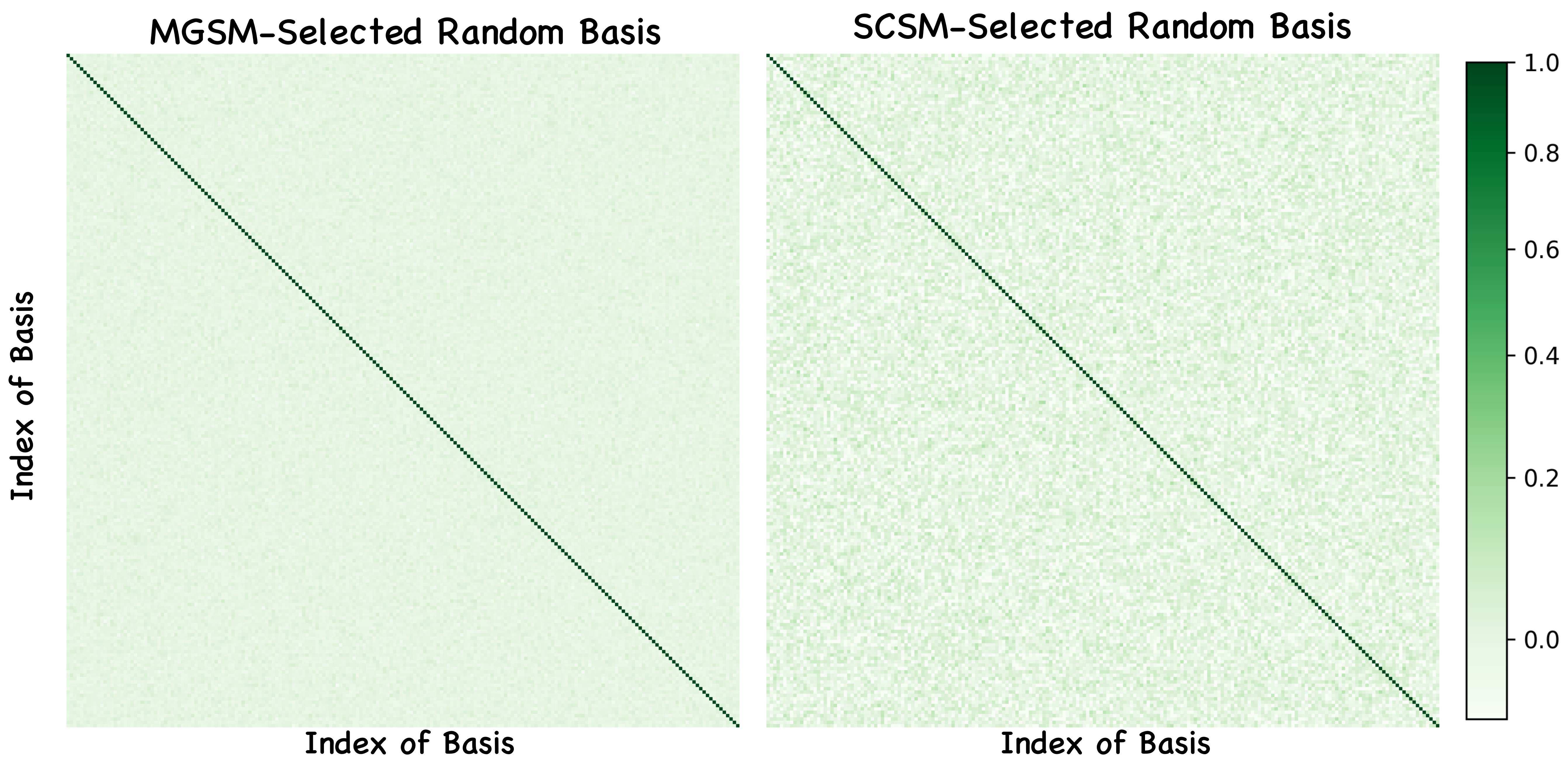}
        \caption{\textbf{Comparison of basis cosine similarity across different strategies.}}
        \label{fig:heatmap}
    \end{minipage}
    \hfill
    \begin{minipage}[t]{0.18\textwidth}
        \centering
        \includegraphics[height=0.16\textheight,width=\linewidth,keepaspectratio]{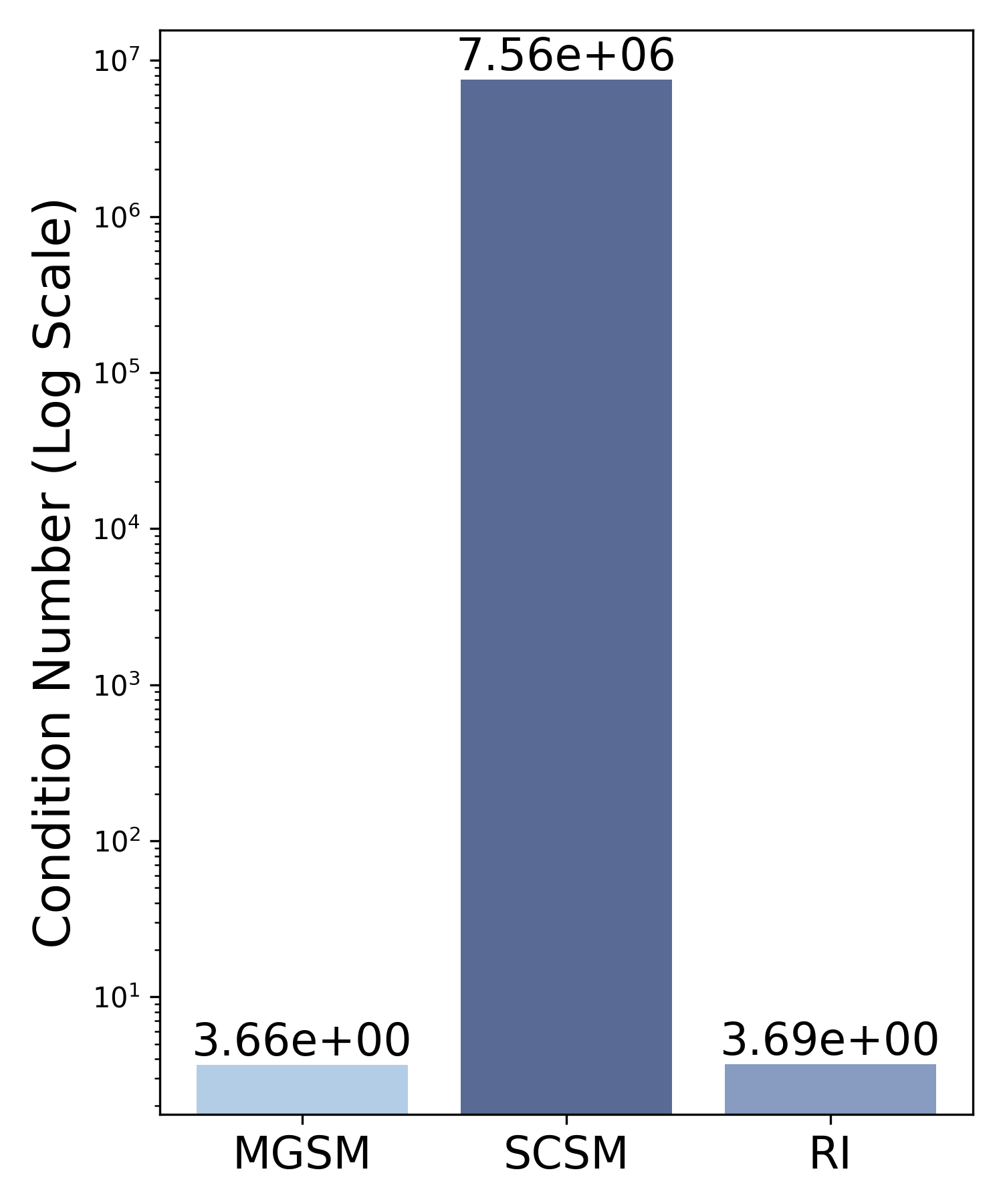}
        \caption{\textbf{Comparison of Condition Numbers.}}
        \label{fig:condition_number}
    \end{minipage}
    \hfill
    \begin{minipage}[t]{0.34\textwidth}
        \centering
        \includegraphics[height=0.16\textheight,width=\linewidth,keepaspectratio]{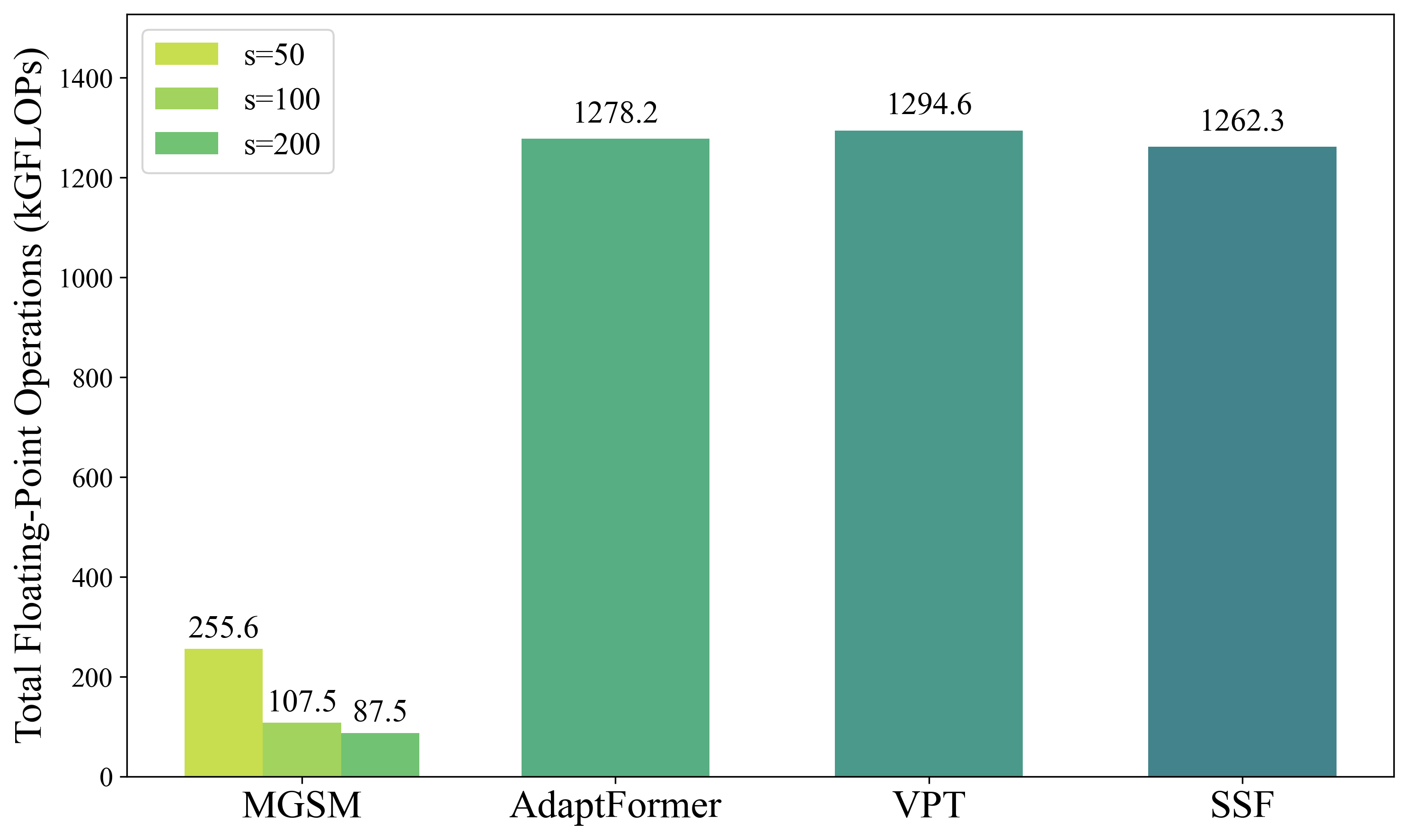}
        \caption{\textbf{Comparison of Computational Cost.}}
        \label{fig:gflops}
    \end{minipage}

\end{figure*}

\begin{figure*}[!t]
\centering
\begin{minipage}[t]{0.35\textwidth}
\vspace{0pt}
\centering
{\captionsetup{position=top}
\captionof{table}{\textbf{Performance Comparison of $s$ and $B_{max}$ on ImageNet-R (B-0 Inc-5).}}
\label{tab:s}}

\begin{small}
\begin{sc}
\resizebox{\linewidth}{!}{
\begin{tabular}{lcccc}
\toprule
\textbf{Method} & \textbf{$s$}& \textbf{$B_{max}$} & $A_{\mathrm{Avg}}$(\%) & \textbf{TIME (s)} \\
\midrule
\multirow{6}{*}{MGSM} & \multirow{3}{*}{10} &5 &75.29±0.15 &92.37±1.34 \\
& &10 &75.31±0.21 &192.21±2.64 \\
& &20 &\textbf{75.90±0.18} &262.56±1.45 \\
\cmidrule(lr){2-5}
& \multirow{3}{*}{100} &5 &74.82±0.17 &63.32±0.81 \\
& &10 &74.86±0.21 &70.37±1.59 \\
& &20 &75.17±0.29 &82.45±2.33 \\
\midrule
\multirow{3}{*}{SCSM} & \multirow{3}{*}{1} &5 &69.84±0.34 &268.61±2.86 \\
& &10 &69.89±0.47 &276.31±2.79 \\
& &20 &70.51±0.29 &332.18±2.12 \\
\bottomrule
\end{tabular}}
\end{sc}
\end{small}
\end{minipage}\hfill
\begin{minipage}[t]{0.63\textwidth}
\vspace{0pt}
\centering
\includegraphics[width=\linewidth]{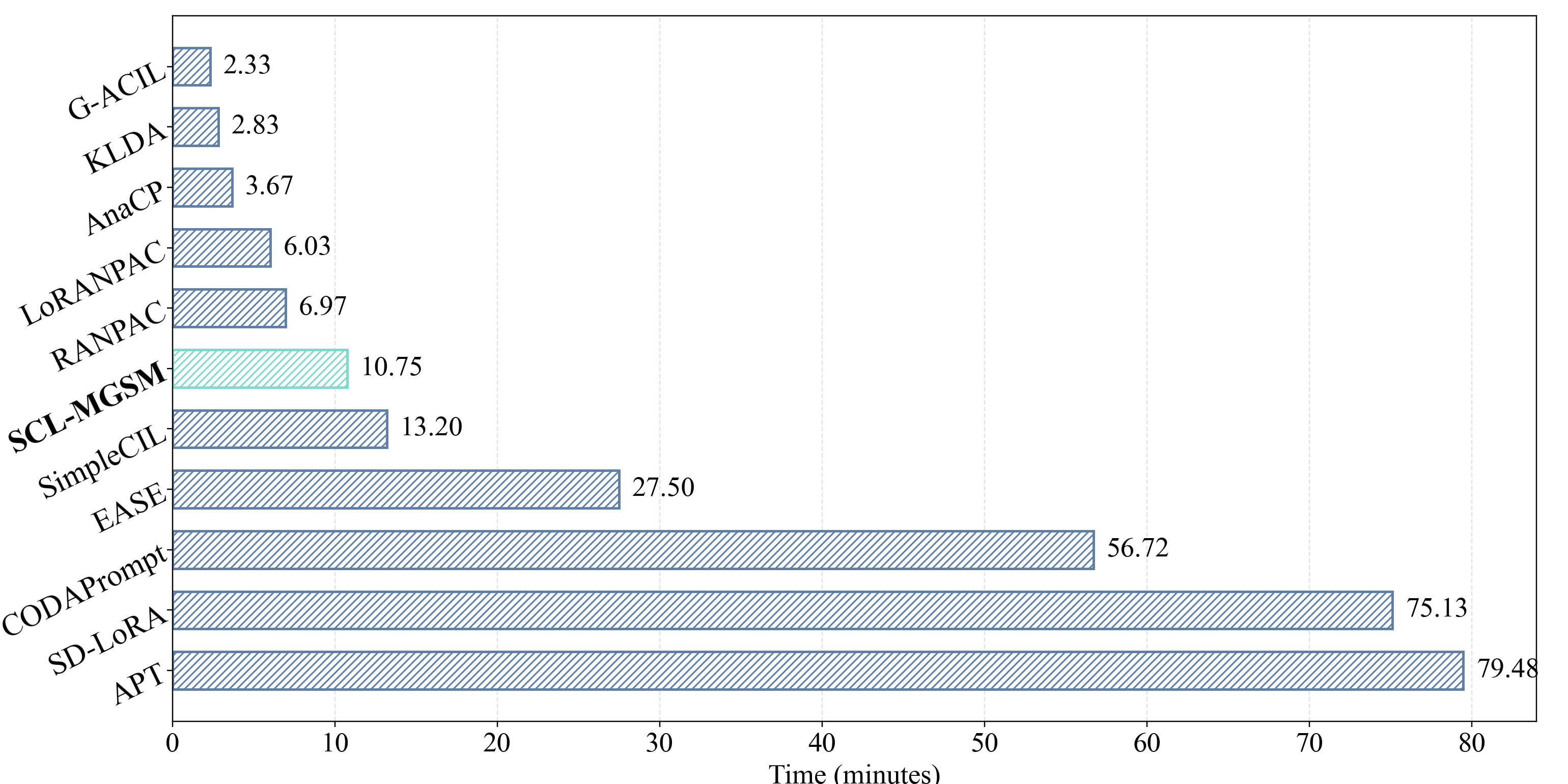}
\vspace{-1.8em}
{\captionsetup{position=bottom}
\captionof{figure}{\textbf{Comparison of Training Time Across Representative CIL Methods.}}
\label{fig:time_consume}}
\end{minipage}

\end{figure*}

\subsection{Hyperparameter Sensitivity}
\noindent\textbf{Analysis of adaptive scaling of $\xi$.}
Figure~\ref{fig:xi} compares two fixed $\xi$ values (0.08 and 0.008) with the adaptive (ADP) strategy on four datasets.
$\xi$ is a scaling factor that controls the sampling range for candidate random bases during RPL construction in SCL-MGSM.
Both fixed $\xi$ settings yield weaker results than the adaptive strategy (ADP), which outperforms them across all four benchmarks.
A fixed $\xi$ keeps the sampling range unchanged throughout the initial construction phase. As the RPL grows and its modeling capacity increases, candidates drawn from the same fixed range become less likely to pass the MGSM acceptance criterion, limiting further effective expansion of the RPL.
The adaptive strategy adjusts $\xi$ according to the current RPL state, so that later construction iterations can sample from a range that remains compatible with the MGSM criterion. This allows SCL-MGSM to explore a broader set of function spaces during construction, producing an RPL that better adapts to diverse downstream tasks.

\noindent\textbf{Analysis of sensitivity to $s$ and $B_{\max}$.}
In Table~\ref{tab:s}, $B_{\text{max}}$ is the maximum number of Gaussian-sampled batches, and $s$ is the number of hidden units per batch. Larger $s$ implies more concurrently integrated units, accelerating convergence.
The same table indicates that a tenfold variation in $s$ negligibly affects overall performance, allowing $s$ to be increased for substantially faster configuration. Concurrently, a larger $B_{\text{max}}$ enables broader function space exploration, improving optimal hidden unit selection probability. Furthermore, SCSM is significantly slower than MGSM, primarily because: first, its unit-by-unit configuration requires more forward computations, and second, its Moore-Penrose pseudoinverse calculation is computationally expensive.

\subsection{Additional Analysis}

\noindent\textbf{Computational efficiency.}
Figure~\ref{fig:time_consume} reports the end-to-end training time on ImageNet-R (Inc-10). RPL-based methods achieve the best CIL performance while being generally faster than other paradigms. Among them, SCL-MGSM introduces only a minor overhead over other RPL-based baselines, in exchange for notably improved performance.
The additional cost of MGSM stems solely from the increased number of recursive inverse updates during RPL construction (Figure~\ref{fig:gflops}), which remains far below PEFT. In the extreme case, one can pair a frozen PTM directly with MGSM, bypassing FSA entirely, to eliminate the adaptation cost while still benefiting from data-guided RPL construction. Additional analysis is provided in the appendix.
Moreover, during incremental learning, compared with directly enlarging the hidden dimension to improve expressivity, SCL-MGSM is significantly more memory-efficient. For example, constructing a 100{,}000-dimensional RPL via RI and employing the stabilized recursive solver of~\cite{peng2024loranpac} for incremental updates demands tens of times more peak GPU memory across CIL stages (Figure~\ref{fig:mem}), indicating that improving performance by unguidedly enlarging the RPL quickly encounters resource bottlenecks.

\begin{figure}[!t]
\centering
\begin{minipage}[t]{0.48\linewidth}
    \centering
    \includegraphics[width=\linewidth]{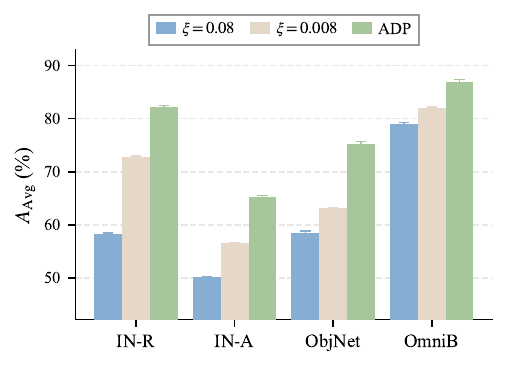}
    \caption{\textbf{Parameter Analysis of MGSM on Three Datasets.} ADP refers to adaptive $\xi$ adjustment.}
    \label{fig:xi}
\end{minipage}
\hfill
\begin{minipage}[t]{0.48\linewidth}
    \centering
    \includegraphics[width=\linewidth]{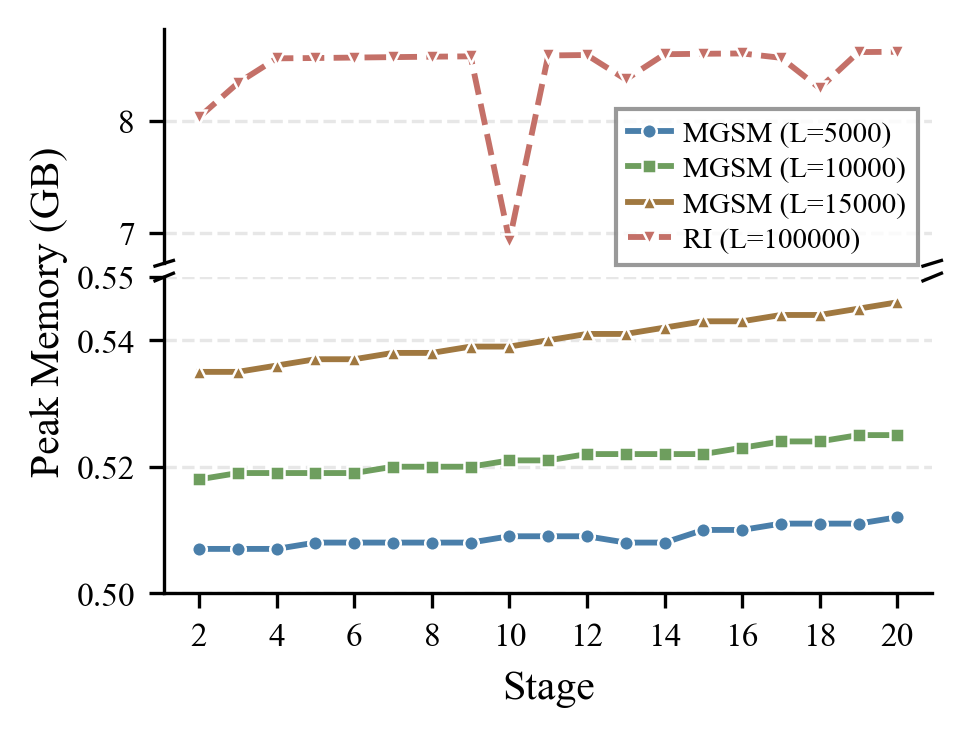}
    \caption{\textbf{Comparison of Peak GPU Memory.}}
    \label{fig:mem}
\end{minipage}
\vskip -0.2in
\end{figure}

\noindent\textbf{More analysis.} We present additional analysis in the appendix, including 1) the efficiency of SCL-MGSM across different PTM architectures, 2) the impact of first-session adaptation on RPL-based methods, and 3) the comparison of average forgetting across methods.

\section{Conclusion}
In this work, we propose SCL-MGSM to enhance pretrained model-based continual representation learning via guided random projection. MGSM employs a target-aligned residual criterion to progressively select informative and non-redundant random bases, constructing a compact yet expressive RPL whose dimension is adaptively determined rather than fixed a priori. On this well-conditioned random feature space, the analytic classifier is updated via recursive ridge regression without external stabilization mechanisms. Extensive experiments on seven exemplar-free CIL benchmarks demonstrate that SCL-MGSM achieves superior performance and training efficiency. Moreover, SCL-MGSM is compatible with both first-session adaptation and diverse PTM backbones, demonstrating strong adaptability across training protocols and architectural families.

{
    \small
    \bibliographystyle{splncs04}
    \bibliography{ref}
}

\clearpage
\onecolumn
\appendix

\section[Proof of Theorem]{Proof of Theorem \ref{sm}.}
\label{proof of MGSM}

\begin{proof}
Throughout, we assume \(\lambda>0\) and take
\(\boldsymbol{W}_{\beta_{L-s}}\) to be the ridge-regression solution
restricted to the first \(L-s\) hidden units at the current accepted
augmentation step.
Consider the matrix decomposition
\( \boldsymbol{H}_{L}=[ \boldsymbol{H}_{L-s}, \boldsymbol{H}_s] \). The
ridge-regularized Gram matrix is given by:
\begin{equation}
\boldsymbol{H}_L^\top \boldsymbol{H}_L + \lambda \boldsymbol{I}
= \begin{bmatrix}
\boldsymbol{H}_{L-s}^\top \boldsymbol{H}_{L-s} + \lambda \boldsymbol{I}_{L-s} & \boldsymbol{H}_{L-s}^\top \boldsymbol{H}_s\\[4pt]
\boldsymbol{H}_s^\top \boldsymbol{H}_{L-s} & \boldsymbol{H}_s^\top \boldsymbol{H}_s + \lambda \boldsymbol{I}_s
\end{bmatrix}.
\end{equation}
Using the block inversion formula, the lower-right block \( \boldsymbol{S} \) is given by:
\begin{equation}
\label{S}
\begin{split}
\boldsymbol{S} &= \left( \boldsymbol{H}_{s}^{\top} \boldsymbol{H}_{s} + \lambda \boldsymbol{I}_{s} \right) \\
&\quad - \boldsymbol{H}_{s}^{\top} \boldsymbol{H}_{L-s}
\left( \boldsymbol{H}^{\top}_{L-s} \boldsymbol{H}_{L-s} + \lambda \boldsymbol{I}_{L-s} \right)^{-1}
\boldsymbol{H}^{\top}_{L-s} \boldsymbol{H}_{s}.
\end{split}
\end{equation}
Since \(\lambda>0\), both diagonal blocks are positive definite, so \(\boldsymbol{S}\) is positive definite and invertible.
The updated solution for the output weight \( \boldsymbol{W}_{\beta_{L}^{\star}} \) arises from the closed-form solution of ridge regression (Eq. \ref{update_weight}), which, when adapted for the block matrix decomposition $\boldsymbol{H}_L=[\boldsymbol{H}_{L-s}, \boldsymbol{H}_s]$ using standard block matrix inversion techniques (general principles in \cite{tylavsky2005generalization}, with $\boldsymbol{S}$ in Eq. \ref{S} being the key Schur complement component for this inversion), yields the recursive update presented in Eq. \ref{wbl}.
This specific form is a direct consequence of the optimality conditions for block-wise ridge regression and is a well-established result in the context of sequential least squares (for detailed derivations, see e.g., \cite{greville1960some} and \cite{blum2012density}).

\begin{equation}
\label{wbl}
\boldsymbol{W}_{\beta_{L}^{\star}}
= \begin{pmatrix}
\boldsymbol{W}_{\beta_{L-s}} \\[3pt] \boldsymbol{0}
\end{pmatrix}
+ \begin{pmatrix}
\boldsymbol{\Delta} \\[3pt] \boldsymbol{S}^{-1} \boldsymbol{H}_{s}^{\top}
\end{pmatrix}
\bigl( \boldsymbol{y} - \boldsymbol{H}_{L-s} \boldsymbol{W}_{\beta_{L-s}} \bigr),
\end{equation}
where \( \boldsymbol{W}_{\beta_{L-s}} \) is the solution using the first $L-s$ nodes. The first term in Eq. \ref{wbl} extends this previous solution with zero-padding for the newly added $s$ nodes. The second term reflects the correction, based on the residual $\boldsymbol{y} - \boldsymbol{H}_{L-s} \boldsymbol{W}_{\beta_{L-s}}$. Specifically, the component $\boldsymbol{S}^{-1} \boldsymbol{H}_{s}^{\top}$ computes the optimal update for the weights of the new block $\boldsymbol{H}_s$, while $\boldsymbol{\Delta}$ adjusts the weights corresponding to $\boldsymbol{H}_{L-s}$, and $\boldsymbol{\Delta}$ is given by:
\begin{equation}
\begin{split}
\boldsymbol{\Delta} &= - (\boldsymbol{H}_{L-s}^\top \boldsymbol{H}_{L-s} + \lambda \boldsymbol{I}_{L-s})^{-1} \\
&\quad \boldsymbol{H}_{L-s}^\top \boldsymbol{H}_s \boldsymbol{S}^{-1} \boldsymbol{H}_s^\top.
\end{split}
\end{equation}
When the output weights are updated, the new residual is given by:
\begin{equation}
\boldsymbol{e}_{L} = \boldsymbol{y} - \boldsymbol{H}_{L}\boldsymbol{W}_{\beta_{L}^{\star }}.
\end{equation}
Substituting Eq. \ref{wbl} into \(\boldsymbol{H}_{L}\boldsymbol{W}_{\beta_{L}^{\star }}\), we obtain:
\begin{equation}
\begin{split}
\boldsymbol{H}_{L}\boldsymbol{W}_{\beta_{L}^{\star }} &= \boldsymbol{H}_{L-s}\boldsymbol{W}_{\beta_{L-s}} \\
&\quad + \boldsymbol{H}_{L-s}\boldsymbol{\Delta}\bigl(\boldsymbol{y}-\boldsymbol{H}_{L-s}\boldsymbol{W}_{\beta_{L-s}}\bigr) \\
&\quad + \boldsymbol{H}_{s} \boldsymbol{S}^{-1} \boldsymbol{H}_{s}^{\top} ( \boldsymbol{y}-\boldsymbol{H}_{L-s}\boldsymbol{W}_{\beta_{L-s}}).
\end{split}
\end{equation}
Since \( \boldsymbol{e}_{L-s} = \boldsymbol{y} - \boldsymbol{H}_{L-s}\boldsymbol{W}_{\beta_{L-s}} \), define
\begin{equation}
\boldsymbol{T}:=\boldsymbol{H}_{L-s}(\boldsymbol{H}_{L-s}^{\top}\boldsymbol{H}_{L-s}+\lambda\boldsymbol{I}_{L-s})^{-1}\boldsymbol{H}_{L-s}^{\top}\boldsymbol{H}_{s}.
\end{equation}
Then \( \boldsymbol{H}_{L-s}\boldsymbol{\Delta} = -\boldsymbol{T}\boldsymbol{S}^{-1}\boldsymbol{H}_{s}^{\top} \), and hence
\begin{equation}
\begin{split}
\boldsymbol{e}_{L} &= \boldsymbol{y}-\boldsymbol{H}_{L}\boldsymbol{W}_{\beta_{L}^{\star }} \\
&= ( \boldsymbol{y}-\boldsymbol{H}_{L-s}\boldsymbol{W}_{\beta_{L-s}}) - \boldsymbol{H}_{s} \boldsymbol{S}^{-1} \big( \boldsymbol{H}_{s}^{\top} \boldsymbol{e}_{L-s} \big) \\
&\quad + \boldsymbol{T}\boldsymbol{S}^{-1} \big( \boldsymbol{H}_{s}^{\top} \boldsymbol{e}_{L-s} \big).
\end{split}
\end{equation}
Defining \( \boldsymbol{v} = \boldsymbol{H}_s^\top \boldsymbol{e}_{L-s} \), we obtain:
\begin{equation}
\boldsymbol{e}_{L} = \boldsymbol{e}_{L-s} - \boldsymbol{H}_{s} \boldsymbol{S}^{-1} \boldsymbol{v} + \boldsymbol{T}\boldsymbol{S}^{-1}\boldsymbol{v}.
\end{equation}
For brevity, let \( \boldsymbol{z} = \boldsymbol{H}_s \boldsymbol{S}^{-1} \boldsymbol{v} \) and \( \boldsymbol{u} = \boldsymbol{T}\boldsymbol{S}^{-1}\boldsymbol{v} \). Then, considering the squared norm of the residual difference:
\begin{align}
\|\boldsymbol{e}_{L-s}\|^2 &- \|\boldsymbol{e}_{L}\|^2 = \|\boldsymbol{e}_{L-s}\|^2 - \|\boldsymbol{e}_{L-s} - \boldsymbol{z} + \boldsymbol{u}\|^2 \notag \\
&= \boldsymbol{e}_{L-s}^\top \boldsymbol{z} + \boldsymbol{z}^\top \boldsymbol{e}_{L-s} - \boldsymbol{z}^\top \boldsymbol{z} \notag \\
&\quad - \boldsymbol{e}_{L-s}^\top \boldsymbol{u} - \boldsymbol{u}^\top \boldsymbol{e}_{L-s} + \boldsymbol{z}^\top \boldsymbol{u} + \boldsymbol{u}^\top \boldsymbol{z} - \boldsymbol{u}^\top \boldsymbol{u}.
\end{align}
Next, we compute each principal term:
1. Since \( \boldsymbol{v} = \boldsymbol{H}_s^{\top} \boldsymbol{e}_{L-s} \), we have:
\begin{equation}
\begin{split}
\boldsymbol{e}_{L-s}^\top \boldsymbol{z} &= \boldsymbol{e}_{L-s}^\top \bigl(\boldsymbol{H}_s \boldsymbol{S}^{-1} \boldsymbol{v}\bigr) = \bigl(\boldsymbol{H}_s^\top \boldsymbol{e}_{L-s}\bigr)^\top \boldsymbol{S}^{-1} \boldsymbol{v} = \boldsymbol{v}^\top \boldsymbol{S}^{-1} \boldsymbol{v}.
\end{split}
\end{equation}
2. Similarly, for \( \boldsymbol{z}^\top \boldsymbol{e}_{L-s} \):
\begin{equation}
\boldsymbol{z}^\top \boldsymbol{e}_{L-s} = \boldsymbol{v}^\top \boldsymbol{S}^{-1} \boldsymbol{v}.
\end{equation}
3. For \( \boldsymbol{z}^\top \boldsymbol{z} \):
\begin{equation}
\boldsymbol{z}^\top \boldsymbol{z} = \boldsymbol{v}^\top \boldsymbol{S}^{-1} \bigl(\boldsymbol{H}_s^\top \boldsymbol{H}_s\bigr) \boldsymbol{S}^{-1} \boldsymbol{v}.
\end{equation}
Thus, we obtain:
\begin{equation}
\label{eq:residual_decomp}
\begin{split}
\|\boldsymbol{e}_{L-s}\|^2 - \|\boldsymbol{e}_{L}\|^2 &= 2 \boldsymbol{v}^{\top} \boldsymbol{S}^{-1} \boldsymbol{v} - \boldsymbol{v}^{\top} \Bigl( \boldsymbol{S}^{-1} \boldsymbol{H}_{s}^{\top} \boldsymbol{H}_{s} \boldsymbol{S}^{-1} \Bigr) \boldsymbol{v} + \mathcal{R}_{L},
\end{split}
\end{equation}
where
\begin{equation}
\label{eq:coupling_term}
\mathcal{R}_{L}:=- \boldsymbol{e}_{L-s}^\top \boldsymbol{u} - \boldsymbol{u}^\top \boldsymbol{e}_{L-s} + \boldsymbol{z}^\top \boldsymbol{u} + \boldsymbol{u}^\top \boldsymbol{z} - \boldsymbol{u}^\top \boldsymbol{u}.
\end{equation}
The term \(\mathcal{R}_{L}\) denotes the coupling terms introduced by the
re-adjustment of the previously selected block \(\boldsymbol{H}_{L-s}\) through
the Schur-complement correction \(\boldsymbol{T}\boldsymbol{S}^{-1}\boldsymbol{v}\);
it is absent in the idealized decoupled case.
Since \(\mathcal{R}_{L}\ge 0\) by assumption, Eq.~(\ref{eq:residual_decomp}) and condition~(\ref{eq:core_sm}) imply:
\begin{equation}
\| \boldsymbol{e}_{L} \|^{2} \le r \, \| \boldsymbol{e}_{L-s} \|^{2}.
\end{equation}
Let \(\boldsymbol{e}_{0}\) denote the residual before the first accepted block in this construction. Since each accepted expansion appends \(s\) hidden units, after \(k\) accepted block expansions we have \(L=ks\), and therefore
\begin{equation}
\| \boldsymbol{e}_{L} \|^{2} \le r^{k} \, \| \boldsymbol{e}_{0} \|^{2}.
\end{equation}
Since \(0<r<1\), \(r^{k}\to 0\) as \(k\to\infty\), and hence
\(\lim_{L\to\infty}\|\boldsymbol{e}_L\|=0\).
This completes the proof.
\end{proof}

\clearpage
\section{Explanation of the MGSM-driven RPL Modeling Process}
\label{a_MGSM_process}

\paragraph{Process overview.}

Fig.~\ref{process} illustrates the overall MGSM-driven RPL modeling process.
We refer to each randomly parameterized hidden unit $(\boldsymbol{w}, b)$ in the RPL as a random basis; its activation $g(\boldsymbol{w}^\top \boldsymbol{z} + b)$ on a PTM feature vector $\boldsymbol{z}$ contributes one column to the RPL feature matrix.
These random bases are continually sampled from simple distributions at each iteration, but only those whose induced features satisfy the MGSM acceptance criterion are incorporated into the final RPL. The goal is to ensure that the resulting random features lie as much as possible within the domain of the downstream task, for which the initial-task dataset provides the best available reference at the current stage. Consequently, MGSM can be viewed as a data- and objective-driven exploration procedure: when no candidate random bases drawn from the current sampling distribution are accepted, MGSM automatically adjusts the distribution range so that more sampled bases fall into the admissible region defined by its criterion. This process continues until the residual falls below a predefined threshold, at which point the RPL construction terminates.

Although the acceptance criterion in Theorem~\ref{sm} does not formally guarantee that a candidate block will be accepted within a finite number of trials, in all of our experiments across four benchmarks, multiple PTM architectures, and three independent runs each, at least one candidate block was always accepted before the discrete scaling set $\mathcal{X}_{\xi}$ was exhausted. As $\xi$ increases, the sampled random features span increasingly diverse directions, making it progressively easier for a candidate batch to satisfy the criterion.

\begin{figure*}[h]
  \vskip -0.2in
	\begin{center}\centerline{\includegraphics[width=0.8\textwidth]{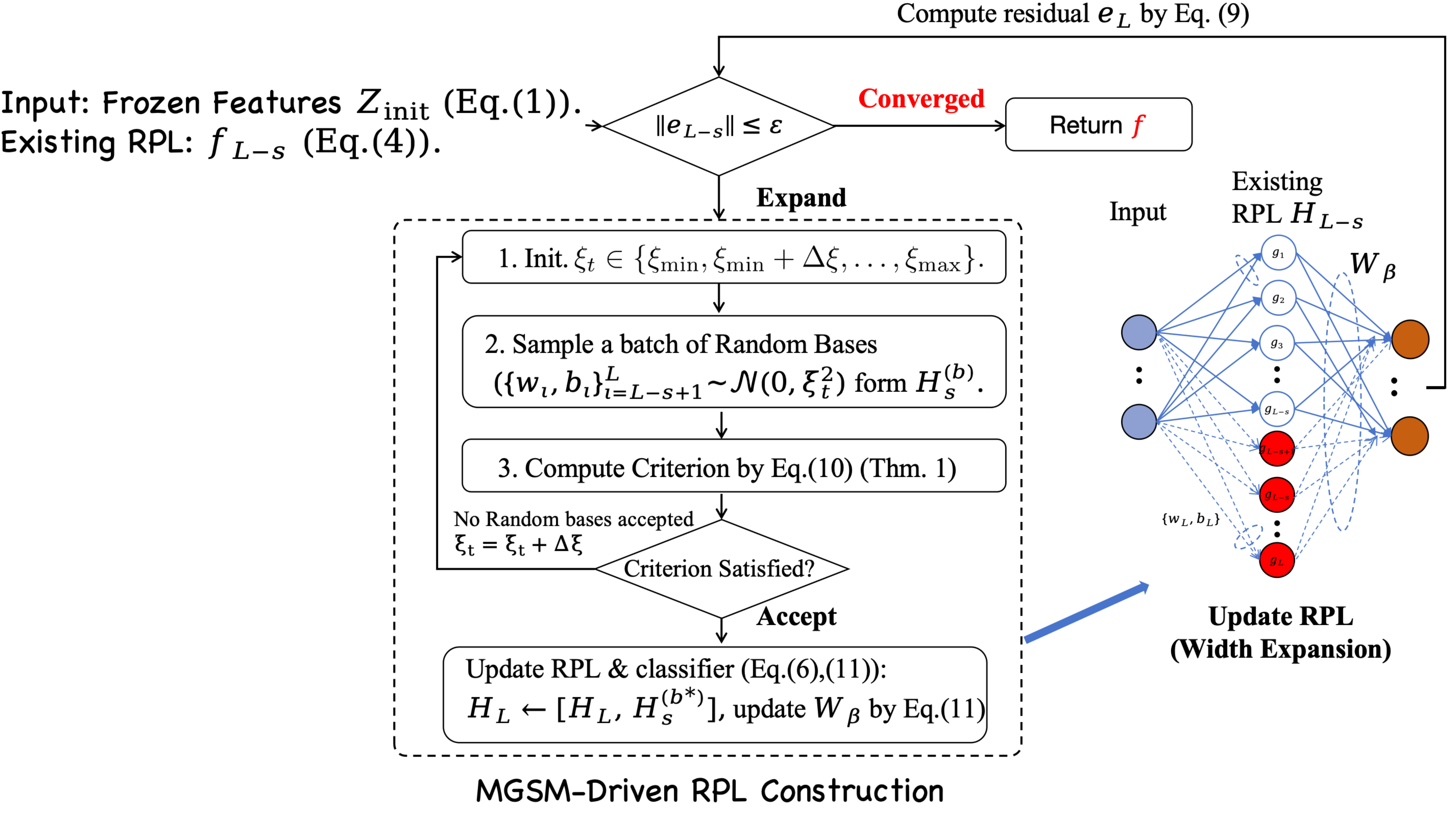}}
		\caption{\textbf{Illustration of the MGSM-driven RPL construction process.}}
		\label{process}
	\end{center}
\end{figure*}

\clearpage
\section{Additional Experiment}
\label{a_details}

\subsection{Hyperparameters}
\label{appendix_hyper}

\noindent\textbf{Our Method.}
The proposed SCL-MGSM involves the following hyperparameters, all of which are shared across the four benchmarks.
The activation function $g(\cdot)$ is sigmoid.
The batch size $s{=}50$ and maximum number of candidate batches $B_{\max}{=}10$ mainly affect the construction speed (see Table~\ref{tab:s}).
The contraction rate in Theorem~\ref{sm} is fixed at $r{=}0.99$.
The residual tolerance is $\varepsilon{=}0.01$.
The ridge regularization coefficient $\lambda{=}0.01$ is selected by grid search over $\{0.001, 0.01, 0.1, 10, 100, 1000\}$ with a 90\%-10\% train-validation split on the initial-task data.
The adaptive scaling range is $\xi_{\min}{=}0.0008$, $\Delta\xi{=}0.0001$, $\xi_{\max}{=}0.004$.
For first-session adaptation (FSA), we follow the same protocol as RanPAC~\cite{mcdonnell2024ranpac}.

\noindent\textbf{Baselines.}
All compared methods are reproduced using their official codebases.
For KLDA, APT, CodaPrompt, and SD-LoRA, we follow the ImageNet-R hyperparameter settings.
Methods using first-session adaptation, including LoRanPAC, adopt the same hyperparameters as RanPAC~\cite{mcdonnell2024ranpac}.
All other methods use their default hyperparameters.
All experiments are conducted on a single Nvidia A100 GPU with 80GB memory.

\subsection{Final Hidden Size of SCL-MGSM}
\label{appendix_hs}
Unlike fixed-size RPL methods that require manual tuning of the hidden size, SCL-MGSM determines $L^*$ automatically via the MGSM residual criterion. Table~\ref{tab:final_nodes} reports the average final hidden size over three runs. The results show that $L^*$ varies across datasets, reflecting differences in domain shift: ImageNet-A, which exhibits the largest domain gap from the pre-trained distribution, requires the most hidden units ($L^*{=}15950$), whereas ImageNet-R converges with the fewest ($L^*{=}14650$). This data-driven capacity allocation enhances usability by automating the selection of hidden sizes for different tasks, obviating the need for manual searching.

\begin{table}[H]
\centering
\caption{\textbf{Average final hidden sizes ($L^*$) of SCL-MGSM across datasets.}}
\label{tab:final_nodes}
\scriptsize
\begin{sc}
\begin{tabular}{lcc}
\toprule
\textbf{Dataset} & \textbf{Avg. Final Nodes ($L^*$)} & \textbf{Std. Dev.} \\
\midrule

ImageNet-R & 14650 & ± 106 \\
ImageNet-A  & 15950 & ± 313 \\
ObjectNet   & 15250 & ± 176 \\
OmniBenchmark & 15450 & ± 248 \\
\bottomrule
\end{tabular}
\end{sc}
\end{table}

\subsection{Impact of First-Session Adaptation on RPL-based Methods}
\label{appendix_fsa_impact}
To isolate the effect of first-session adaptation (FSA), Fig.~\ref{fig:fsa_domain_gap} compares the gain in $A_{\text{last}}$ over the corresponding \textsc{Joint Linear Probe} baseline for representative RPL-based methods under matched evaluation settings. Dark bars denote the variants with FSA, and light bars denote the counterparts without FSA. Positive values indicate that the incremental method surpasses the frozen-backbone oracle that jointly trains only the linear head on all tasks.

Across most datasets, enabling FSA enlarges the margin over \textsc{Joint Linear Probe} for RPL-based methods, indicating that adapting the PTM on the initial task yields representations that better capture the statistical structure of downstream classes, thereby mitigating the domain gap between pre-training and target distributions. Furthermore, even without FSA, SCL-MGSM consistently outperforms the other RPL-based baselines, suggesting that MGSM-guided subspace modeling itself can partially compensate for the domain gap by leveraging the initial-task statistics to construct a more representative RPL. When FSA is enabled, SCL-MGSM benefits consistently on all four datasets and achieves the largest gains among the compared methods, indicating that the adapted first-session representation and MGSM-guided RPL construction are complementary and can be combined for stronger continual performance.

\begin{figure}[H]
    \centering
    \includegraphics[width=0.96\textwidth]{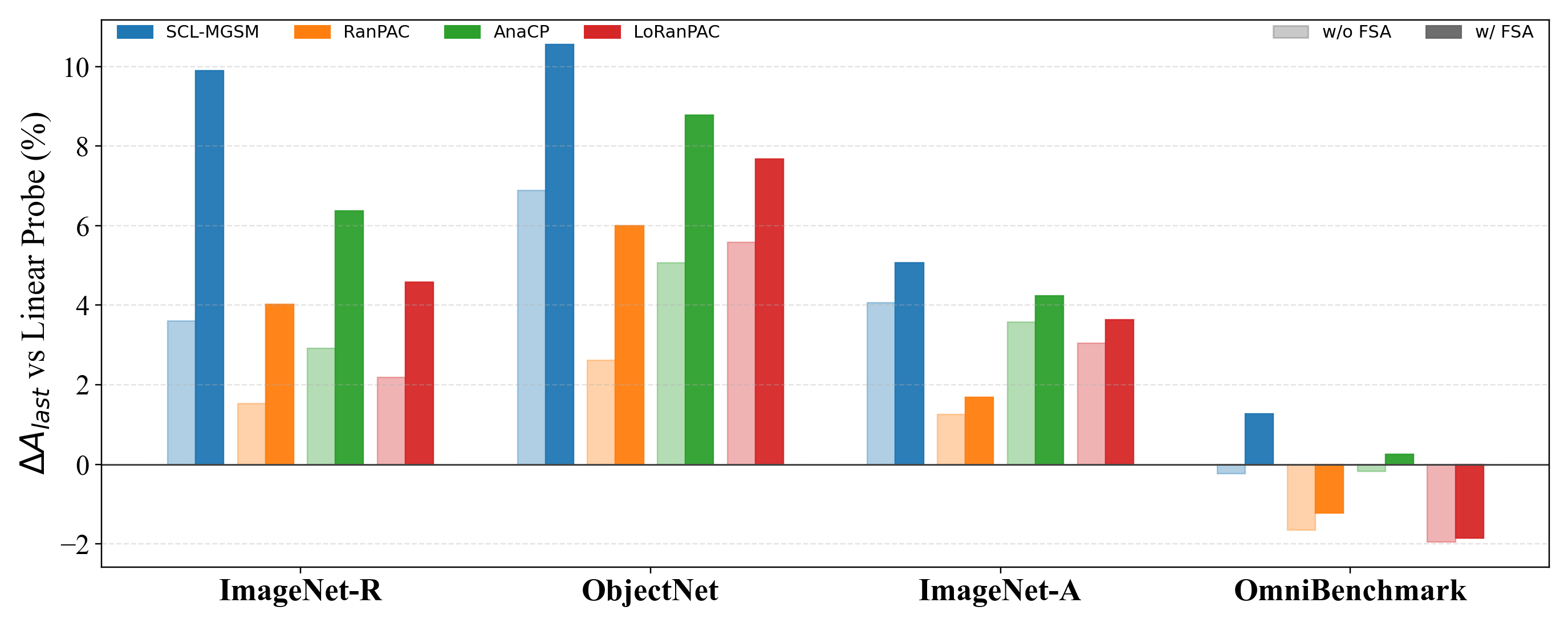}
    \caption{\textbf{Impact of first-session adaptation on representative RPL-based methods.} We report $\Delta A_{\text{last}}$ relative to the corresponding \textsc{Joint Linear Probe} baseline on four datasets. Dark bars indicate variants with FSA, and light bars indicate variants without FSA. Larger positive values mean that the incremental learner exceeds the frozen-backbone joint linear probe by a wider margin.}
    \label{fig:fsa_domain_gap}
\end{figure}

\subsection{Result on More Backbones}
\label{appendix_morebackbone}

\paragraph{Result on DINO-v2.}
\label{appendix_dinov2}
As shown in Table~\ref{tab:dinov2}, with DINO-v2~\cite{oquab2023dinov2} as the PTM, SCL-MGSM achieves the best $A_{\text{last}}$ and $A_{\text{avg}}$ on ImageNet-R, OmniBenchmark, and ObjectNet under both protocols, and attains the highest $A_{\text{last}}$ on ImageNet-A under both protocols. Specifically, on OmniBenchmark Inc-5, SCL-MGSM surpasses the second-best method (AnaCP) by +2.08\% in $A_{\text{last}}$, and on ImageNet-R Inc-10 the margin reaches +0.46\% over AnaCP. The only exception is ImageNet-A $A_{\text{avg}}$ Inc-5, where AnaCP leads by a small margin. Notably, on OmniBenchmark SCL-MGSM surpasses the \textsc{Joint Linear Probe} oracle (83.40\% vs.\ 82.41\%), despite operating in a strictly incremental setting. Relative to the corresponding results in Table~\ref{tab:all_results} obtained with ViT-B/16-IN21K, DINO-v2 consistently improves all compared methods. SCL-MGSM nevertheless retains its advantage over the other RPL-based baselines, indicating that MGSM-guided subspace modeling continues to benefit from stronger feature representations rather than saturating as the backbone improves.

\paragraph{Result on non-Transformer Backbone.}
As shown in Table~\ref{tab:cnn}, when replacing the Transformer backbone with ResNet-101, SCL-MGSM ranks first on ImageNet-R, ObjectNet, and OmniBenchmark across both incremental protocols. Specifically, on OmniBenchmark Inc-5, SCL-MGSM outperforms the second-best method (RanPAC) by +3.00\% in $A_{\text{last}}$ and +2.61\% in $A_{\text{avg}}$. The exception is ImageNet-A, where LoRanPAC achieves the best results and SCL-MGSM ranks second. A possible reason is that ImageNet-A contains adversarially filtered samples that are inherently challenging for CNN features, limiting the benefit that subspace modeling can extract from relatively weak representations. Nevertheless, across the remaining three datasets SCL-MGSM consistently leads, demonstrating that MGSM-guided RPL construction is architecture-agnostic and not restricted to Transformer-based PTMs. These results confirm that SCL-MGSM serves as a plug-and-play module compatible with diverse PTM architectures.

\begin{table*}[t]
\caption{\textbf{Comparison of CIL performance across methods with DINO-v2 as the PTM.} B-$m$ Inc-$n$ denotes $m$ initial classes and $n$ incremental classes (\textit{$m{=}0$}: equal split). $A_{\text{last}}$ and $A_{\text{avg}}$ denote final and average accuracy (\%), respectively. \textbf{Bold} and \underline{underline} mark the best and second-best incremental results under each protocol.}
\label{tab:dinov2}
\centering
\setlength{\tabcolsep}{2pt}
\scriptsize
\resizebox{\textwidth}{!}{
\begin{tabular}{l*{8}{c}}
\toprule
\textbf{Method} &
\multicolumn{4}{c}{\textbf{\makecell{ImageNet-R  (B-0)}}} &
\multicolumn{4}{c}{\textbf{\makecell{ImageNet-A  (B-0)}}} \\
\cmidrule(lr){2-5} \cmidrule(lr){6-9}
Joint Linear Probe &
\multicolumn{4}{c}{82.47$\pm$0.07} &
\multicolumn{4}{c}{69.67$\pm$0.37} \\
\cmidrule(lr){1-9}
&
\multicolumn{2}{c}{\textbf{Inc-5}} &
\multicolumn{2}{c}{\textbf{Inc-10}} &
\multicolumn{2}{c}{\textbf{Inc-5}} &
\multicolumn{2}{c}{\textbf{Inc-10}} \\
\cmidrule(lr){2-3} \cmidrule(lr){4-5} \cmidrule(lr){6-7} \cmidrule(lr){8-9}
&
\textbf{$A_{\text{last}}$} & \textbf{$A_{\text{avg}}$} &
\textbf{$A_{\text{last}}$} & \textbf{$A_{\text{avg}}$} &
\textbf{$A_{\text{last}}$} & \textbf{$A_{\text{avg}}$} &
\textbf{$A_{\text{last}}$} & \textbf{$A_{\text{avg}}$} \\
\midrule
SimpleCIL &
75.35$\pm$0.00 & 80.82$\pm$0.94 & 75.35$\pm$0.00 & 80.52$\pm$0.91 &
68.93$\pm$0.00 & 77.58$\pm$1.13 & 68.93$\pm$0.00 & 77.13$\pm$1.07 \\
KLDA &
75.78$\pm$0.17 & 80.76$\pm$1.15 & 75.78$\pm$0.16 & 80.79$\pm$0.78 &
66.91$\pm$0.63 & 70.13$\pm$1.76 & 66.91$\pm$0.63 & 70.15$\pm$1.49 \\
LoRanPAC &
84.06$\pm$0.10 & 88.34$\pm$0.55 & 84.97$\pm$0.21 & 89.11$\pm$0.55 &
66.69$\pm$0.65 & 73.60$\pm$1.08 & 69.23$\pm$0.44 & 75.28$\pm$1.12 \\
AnaCP &
\underline{84.41$\pm$0.40} & \underline{88.93$\pm$0.48} & \underline{85.94$\pm$0.33} & \underline{89.82$\pm$0.42} &
\underline{71.05$\pm$0.72} & \textbf{79.45$\pm$0.72} & \underline{71.98$\pm$0.68} & \underline{78.28$\pm$1.18} \\
RanPAC &
83.67$\pm$0.25 & 87.82$\pm$0.49 & 84.56$\pm$0.41 & 88.46$\pm$0.69 &
48.56$\pm$2.56 & 57.55$\pm$4.60 & 64.65$\pm$2.23 & 70.98$\pm$1.28 \\
G-ACIL &
81.53$\pm$0.19 & 86.33$\pm$0.71 & 81.53$\pm$0.19 & 86.12$\pm$0.67 &
67.61$\pm$0.58 & 75.91$\pm$0.88 & 67.63$\pm$0.60 & 75.45$\pm$0.80 \\
SCL-MGSM &
\textbf{84.61$\pm$0.79} & \textbf{89.26$\pm$0.17} & \textbf{86.40$\pm$0.16} & \textbf{90.18$\pm$0.45} &
\textbf{71.60$\pm$0.42} & \underline{78.66$\pm$0.90} & \textbf{72.42$\pm$0.52} & \textbf{79.30$\pm$1.01} \\
\midrule
\addlinespace[0.3em]
\textbf{Method} &
\multicolumn{4}{c}{\textbf{\makecell{ObjectNet  (B-0)}}} &
\multicolumn{4}{c}{\textbf{\makecell{OmniBenchmark  (B-0)}}} \\
\cmidrule(lr){2-5} \cmidrule(lr){6-9}
Joint Linear Probe &
\multicolumn{4}{c}{65.17$\pm$0.28} &
\multicolumn{4}{c}{82.41$\pm$0.05} \\
\cmidrule(lr){1-9}
&
\multicolumn{2}{c}{\textbf{Inc-5}} &
\multicolumn{2}{c}{\textbf{Inc-10}} &
\multicolumn{2}{c}{\textbf{Inc-5}} &
\multicolumn{2}{c}{\textbf{Inc-10}} \\
\cmidrule(lr){2-3} \cmidrule(lr){4-5} \cmidrule(lr){6-7} \cmidrule(lr){8-9}
&
\textbf{$A_{\text{last}}$} & \textbf{$A_{\text{avg}}$} &
\textbf{$A_{\text{last}}$} & \textbf{$A_{\text{avg}}$} &
\textbf{$A_{\text{last}}$} & \textbf{$A_{\text{avg}}$} &
\textbf{$A_{\text{last}}$} & \textbf{$A_{\text{avg}}$} \\
\midrule
SimpleCIL &
56.13$\pm$0.00 & 66.33$\pm$2.55 & 56.13$\pm$0.00 & 65.84$\pm$2.43 &
73.48$\pm$0.00 & 81.02$\pm$1.06 & 73.48$\pm$0.00 & 80.79$\pm$1.04 \\
KLDA &
59.99$\pm$0.19 & 69.66$\pm$0.68 & 59.99$\pm$0.18 & 67.91$\pm$1.67 &
77.50$\pm$0.14 & 86.29$\pm$1.40 & 77.51$\pm$0.12 & 85.03$\pm$1.23 \\
LoRanPAC &
67.82$\pm$0.26 & 77.04$\pm$1.91 & 69.72$\pm$0.15 & 78.25$\pm$1.82 &
78.62$\pm$0.24 & 86.64$\pm$0.58 & 79.22$\pm$0.41 & 86.83$\pm$0.56 \\
    AnaCP &
    \underline{71.61$\pm$0.22} & \underline{78.74$\pm$2.10} & \underline{71.76$\pm$0.63} & \underline{79.56$\pm$1.49} &
\underline{81.32$\pm$0.27} & \underline{88.49$\pm$0.55} & \underline{81.46$\pm$0.41} & \underline{88.42$\pm$0.56} \\
RanPAC &
66.95$\pm$0.44 & 75.78$\pm$2.34 & 68.78$\pm$0.46 & 77.10$\pm$1.71 &
79.80$\pm$0.17 & 87.21$\pm$0.70 & 79.96$\pm$0.27 & 87.08$\pm$0.52 \\
G-ACIL &
63.58$\pm$0.19 & 73.20$\pm$2.22 & 63.58$\pm$0.18 & 72.77$\pm$2.15 &
80.14$\pm$0.19 & 87.58$\pm$0.44 & 80.14$\pm$0.19 & 87.39$\pm$0.45 \\
    SCL-MGSM &
    \textbf{72.04$\pm$1.32} & \textbf{80.34$\pm$1.96} & \textbf{72.50$\pm$0.59} & \textbf{80.23$\pm$1.76} &
\textbf{83.40$\pm$0.10} & \textbf{89.38$\pm$0.71} & \textbf{83.43$\pm$0.20} & \textbf{89.34$\pm$0.81} \\
\bottomrule
\end{tabular}}
\vskip -0.05in
\end{table*}

\begin{table*}[t]
\caption{\textbf{Comparison of CIL performance across methods with ResNet101 as the PTM.} B-$m$ Inc-$n$ denotes $m$ initial classes and $n$ incremental classes (\textit{$m{=}0$}: equal split). $A_{\text{last}}$ and $A_{\text{avg}}$ denote final and average accuracy (\%), respectively. \textbf{Bold} and \underline{underline} mark the best and second-best results under each protocol.}
\label{tab:cnn}
\centering
\setlength{\tabcolsep}{2pt}
\scriptsize
\resizebox{\textwidth}{!}{
\begin{tabular}{l*{8}{c}}
\toprule
\textbf{Method} &
\multicolumn{4}{c}{\textbf{\makecell{ImageNet-R  (B-0)}}} &
\multicolumn{4}{c}{\textbf{\makecell{ImageNet-A  (B-0)}}} \\
\cmidrule(lr){2-5} \cmidrule(lr){6-9}
&
\multicolumn{2}{c}{\textbf{Inc-5}} &
\multicolumn{2}{c}{\textbf{Inc-10}} &
\multicolumn{2}{c}{\textbf{Inc-5}} &
\multicolumn{2}{c}{\textbf{Inc-10}} \\
\cmidrule(lr){2-3} \cmidrule(lr){4-5} \cmidrule(lr){6-7} \cmidrule(lr){8-9}
&
\textbf{$A_{\text{last}}$} & \textbf{$A_{\text{avg}}$} &
\textbf{$A_{\text{last}}$} & \textbf{$A_{\text{avg}}$} &
\textbf{$A_{\text{last}}$} & \textbf{$A_{\text{avg}}$} &
\textbf{$A_{\text{last}}$} & \textbf{$A_{\text{avg}}$} \\
\midrule
SimpleCIL &
37.28$\pm$0.00 & 45.72$\pm$1.24 & 37.28$\pm$0.00 & 45.13$\pm$1.07 &
13.63$\pm$0.00 & 23.89$\pm$0.11 & 13.63$\pm$0.00 & 23.14$\pm$0.33 \\
KLDA &
48.96$\pm$0.18 & 55.82$\pm$0.99 & 48.89$\pm$0.40 & 55.08$\pm$1.21 &
15.87$\pm$0.40 & 21.75$\pm$0.62 & 15.80$\pm$0.37 & 20.99$\pm$1.40 \\
LoRanPAC &
\underline{53.38$\pm$0.21} & \underline{62.18$\pm$0.84} & \underline{53.48$\pm$0.25} & \underline{61.75$\pm$0.68} &
\textbf{22.03$\pm$0.45} & \textbf{30.81$\pm$1.01} & \textbf{21.94$\pm$0.60} & \textbf{30.15$\pm$1.24} \\
AnaCP &
52.98$\pm$0.41 & 61.63$\pm$1.00 & 53.03$\pm$0.43 & 61.19$\pm$0.39 &
17.86$\pm$0.17 & 28.49$\pm$0.73 & 17.95$\pm$1.16 & 27.65$\pm$0.89 \\
RanPAC &
52.46$\pm$0.12 & 60.84$\pm$0.61 & 52.87$\pm$0.32 & 60.54$\pm$0.77 &
19.84$\pm$0.55 & 29.00$\pm$1.31 & 19.62$\pm$0.27 & 29.08$\pm$0.71 \\
G-ACIL &
49.38$\pm$0.06 & 58.50$\pm$0.42 & 49.38$\pm$0.07 & 57.99$\pm$0.38 &
20.03$\pm$0.37 & 30.03$\pm$0.71 & 20.03$\pm$0.37 & 29.24$\pm$0.75 \\
SCL-MGSM &
\textbf{54.19$\pm$0.84} & \textbf{62.34$\pm$0.50} & \textbf{54.56$\pm$0.30} & \textbf{62.11$\pm$0.61} &
\underline{20.23$\pm$0.27} & \underline{30.27$\pm$0.63} & \underline{20.08$\pm$0.24} & \underline{29.41$\pm$0.63} \\
\midrule
\addlinespace[0.3em]
\textbf{Method} &
\multicolumn{4}{c}{\textbf{\makecell{ObjectNet  (B-0)}}} &
\multicolumn{4}{c}{\textbf{\makecell{OmniBenchmark  (B-0)}}} \\
\cmidrule(lr){2-5} \cmidrule(lr){6-9}
&
\multicolumn{2}{c}{\textbf{Inc-5}} &
\multicolumn{2}{c}{\textbf{Inc-10}} &
\multicolumn{2}{c}{\textbf{Inc-5}} &
\multicolumn{2}{c}{\textbf{Inc-10}} \\
\cmidrule(lr){2-3} \cmidrule(lr){4-5} \cmidrule(lr){6-7} \cmidrule(lr){8-9}
&
\textbf{$A_{\text{last}}$} & \textbf{$A_{\text{avg}}$} &
\textbf{$A_{\text{last}}$} & \textbf{$A_{\text{avg}}$} &
\textbf{$A_{\text{last}}$} & \textbf{$A_{\text{avg}}$} &
\textbf{$A_{\text{last}}$} & \textbf{$A_{\text{avg}}$} \\
\midrule
SimpleCIL &
29.26$\pm$0.00 & 40.07$\pm$3.17 & 29.25$\pm$0.02 & 39.45$\pm$3.05 &
48.44$\pm$0.00 & 60.14$\pm$1.18 & 48.44$\pm$0.00 & 59.66$\pm$1.26 \\
KLDA &
33.42$\pm$0.34 & 42.14$\pm$3.38 & 33.26$\pm$0.13 & 41.84$\pm$2.96 &
55.84$\pm$0.21 & 67.61$\pm$1.03 & 55.97$\pm$0.29 & 67.17$\pm$1.21 \\
LoRanPAC &
\underline{38.34$\pm$0.31} & \underline{50.40$\pm$2.44} & \underline{38.47$\pm$0.32} & \underline{49.81$\pm$2.33} &
58.86$\pm$0.21 & \underline{71.65$\pm$1.08} & 58.90$\pm$0.19 & 71.28$\pm$1.01 \\
AnaCP &
34.69$\pm$0.51 & 46.59$\pm$3.25 & 34.49$\pm$0.59 & 45.96$\pm$2.62 &
59.18$\pm$0.19 & 70.86$\pm$0.95 & 59.39$\pm$0.20 & 70.58$\pm$1.15 \\
RanPAC &
37.41$\pm$0.16 & 48.72$\pm$2.70 & 37.64$\pm$0.25 & 47.99$\pm$2.61 &
\underline{60.38$\pm$0.08} & 71.48$\pm$0.99 & \underline{60.19$\pm$0.44} & \underline{71.32$\pm$0.80} \\
G-ACIL &
35.06$\pm$0.30 & 46.91$\pm$2.32 & 35.07$\pm$0.31 & 46.33$\pm$2.23 &
57.24$\pm$0.09 & 69.09$\pm$0.96 & 57.24$\pm$0.08 & 68.71$\pm$0.99 \\
SCL-MGSM &
\textbf{38.84$\pm$0.75} & \textbf{51.02$\pm$2.59} & \textbf{39.15$\pm$0.11} & \textbf{50.67$\pm$2.37} &
\textbf{63.38$\pm$0.20} & \textbf{74.09$\pm$0.94} & \textbf{63.22$\pm$0.05} & \textbf{73.78$\pm$0.90} \\
\bottomrule
\end{tabular}}
\vskip -0.05in
\end{table*}

\begin{table*}[t!]
\caption{\textbf{Comparison of Average Forgetting.} We report \(F_{\text{avg}}\) (in \%) for representative incremental methods; lower is better. \textbf{Bold} and \underline{underline} mark the lowest and second-lowest forgetting under each protocol.}
\label{tab:forgetting}
\centering
\setlength{\tabcolsep}{8pt}
\scriptsize
\begin{tabular}{l*{4}{c}}
\toprule
\textbf{Method} &
\multicolumn{2}{c}{\textbf{\makecell{ImageNet-R  (B-0)}}} &
\multicolumn{2}{c}{\textbf{\makecell{ImageNet-A  (B-0)}}} \\
\cmidrule(lr){2-3} \cmidrule(lr){4-5}
&
\textbf{Inc-5} & \textbf{Inc-10} &
\textbf{Inc-5} & \textbf{Inc-10} \\
\midrule

SimpleCIL & 8.06$_{\scriptscriptstyle\pm0.46}$ & 7.73$_{\scriptscriptstyle\pm0.26}$ & 11.75$_{\scriptscriptstyle\pm0.51}$ & 11.59$_{\scriptscriptstyle\pm0.33}$ \\

KLDA & 6.89$_{\scriptscriptstyle\pm0.75}$ & 6.52$_{\scriptscriptstyle\pm0.65}$ & \underline{10.15$_{\scriptscriptstyle\pm0.49}$} & 11.67$_{\scriptscriptstyle\pm0.21}$ \\
RanPAC & 7.11$_{\scriptscriptstyle\pm0.49}$ & 6.34$_{\scriptscriptstyle\pm0.38}$ & 24.83$_{\scriptscriptstyle\pm6.31}$ & 10.53$_{\scriptscriptstyle\pm0.44}$ \\
G-ACIL & 7.63$_{\scriptscriptstyle\pm0.35}$ & 7.21$_{\scriptscriptstyle\pm0.43}$ & 13.41$_{\scriptscriptstyle\pm0.27}$ & 12.36$_{\scriptscriptstyle\pm0.07}$ \\
LoRanPAC & \underline{6.66$_{\scriptscriptstyle\pm0.17}$} & \underline{6.03$_{\scriptscriptstyle\pm0.09}$} & 10.98$_{\scriptscriptstyle\pm0.32}$ & \underline{10.48$_{\scriptscriptstyle\pm0.52}$} \\
AnaCP & 7.01$_{\scriptscriptstyle\pm0.40}$ & 6.31$_{\scriptscriptstyle\pm0.29}$ & 12.12$_{\scriptscriptstyle\pm0.22}$ & 10.66$_{\scriptscriptstyle\pm0.56}$ \\
SCL-MGSM & \textbf{6.42$_{\scriptscriptstyle\pm0.52}$} & \textbf{5.31$_{\scriptscriptstyle\pm0.20}$} & \textbf{9.50$_{\scriptscriptstyle\pm0.59}$} & \textbf{9.90$_{\scriptscriptstyle\pm0.74}$} \\
\midrule
\addlinespace[0.3em]
\textbf{Method} &
\multicolumn{2}{c}{\textbf{\makecell{ObjectNet  (B-0)}}} &
\multicolumn{2}{c}{\textbf{\makecell{OmniBenchmark  (B-0)}}} \\
\cmidrule(lr){2-3} \cmidrule(lr){4-5}
&
\textbf{Inc-5} & \textbf{Inc-10} &
\textbf{Inc-5} & \textbf{Inc-10} \\
\midrule

SimpleCIL & 10.19$_{\scriptscriptstyle\pm0.22}$ & 9.97$_{\scriptscriptstyle\pm0.19}$ & \underline{8.15$_{\scriptscriptstyle\pm0.08}$} & \underline{8.04$_{\scriptscriptstyle\pm0.07}$} \\

KLDA & \underline{9.96$_{\scriptscriptstyle\pm0.23}$} & 9.62$_{\scriptscriptstyle\pm0.19}$ & 8.41$_{\scriptscriptstyle\pm0.18}$ & 8.25$_{\scriptscriptstyle\pm0.12}$ \\
RanPAC & 10.85$_{\scriptscriptstyle\pm0.67}$ & 10.06$_{\scriptscriptstyle\pm0.74}$ & 8.48$_{\scriptscriptstyle\pm0.18}$ & 8.19$_{\scriptscriptstyle\pm0.27}$ \\
G-ACIL & 11.55$_{\scriptscriptstyle\pm0.68}$ & 11.20$_{\scriptscriptstyle\pm0.71}$ & 8.84$_{\scriptscriptstyle\pm0.38}$ & 8.66$_{\scriptscriptstyle\pm0.32}$ \\
LoRanPAC & 10.25$_{\scriptscriptstyle\pm0.18}$ & 9.54$_{\scriptscriptstyle\pm0.40}$ & 8.29$_{\scriptscriptstyle\pm0.34}$ & 8.29$_{\scriptscriptstyle\pm0.22}$ \\
AnaCP & 10.25$_{\scriptscriptstyle\pm0.51}$ & \underline{9.41$_{\scriptscriptstyle\pm0.33}$} & 9.29$_{\scriptscriptstyle\pm0.43}$ & 8.69$_{\scriptscriptstyle\pm0.33}$ \\
SCL-MGSM & \textbf{9.19$_{\scriptscriptstyle\pm0.43}$} & \textbf{8.23$_{\scriptscriptstyle\pm0.71}$} & \textbf{8.01$_{\scriptscriptstyle\pm0.12}$} & \textbf{7.60$_{\scriptscriptstyle\pm0.57}$} \\
\bottomrule
\end{tabular}
\vskip -0.1in
\end{table*}

\subsection{Analysis of Average Forgetting.}
We define the average forgetting as \(F_{\text{avg}}=\frac{1}{T-1}\sum_{j=1}^{T-1}\bigl(\max_{t\in\{j,\dots,T-1\}} a_{t,j}-a_{T,j}\bigr)\), where \(T\) is the total number of stages and \(a_{t,j}\) denotes the test accuracy on task \(j\) after stage \(t\). A lower \(F_{\text{avg}}\) indicates better retention of previously learned knowledge.
Table~\ref{tab:forgetting} reports \(F_{\text{avg}}\) under the same settings as Table~\ref{tab:all_results} for representative RPL-based methods. SCL-MGSM achieves the lowest forgetting in all eight settings, and its margin over the second-lowest method is 0.24/0.72 on ImageNet-R, 0.65/0.58 on ImageNet-A, 0.77/1.18 on ObjectNet, and 0.14/0.44 on OmniBenchmark for Inc-5/Inc-10, respectively. Among the compared RPL-based methods (RanPAC, G-ACIL, LoRanPAC, AnaCP), SCL-MGSM consistently attains the lowest \(F_{\text{avg}}\), confirming that the RPL construction from MGSM preserves previously learned knowledge more effectively than random initialization while maintaining the strongest accuracy in Table~\ref{tab:all_results}. Figure~\ref{fig:seed2025_curves} further visualizes the stage-wise accuracy under the B-0 Inc-10 setting.

\begin{figure*}[t!]
    \centering
    \includegraphics[width=\linewidth]{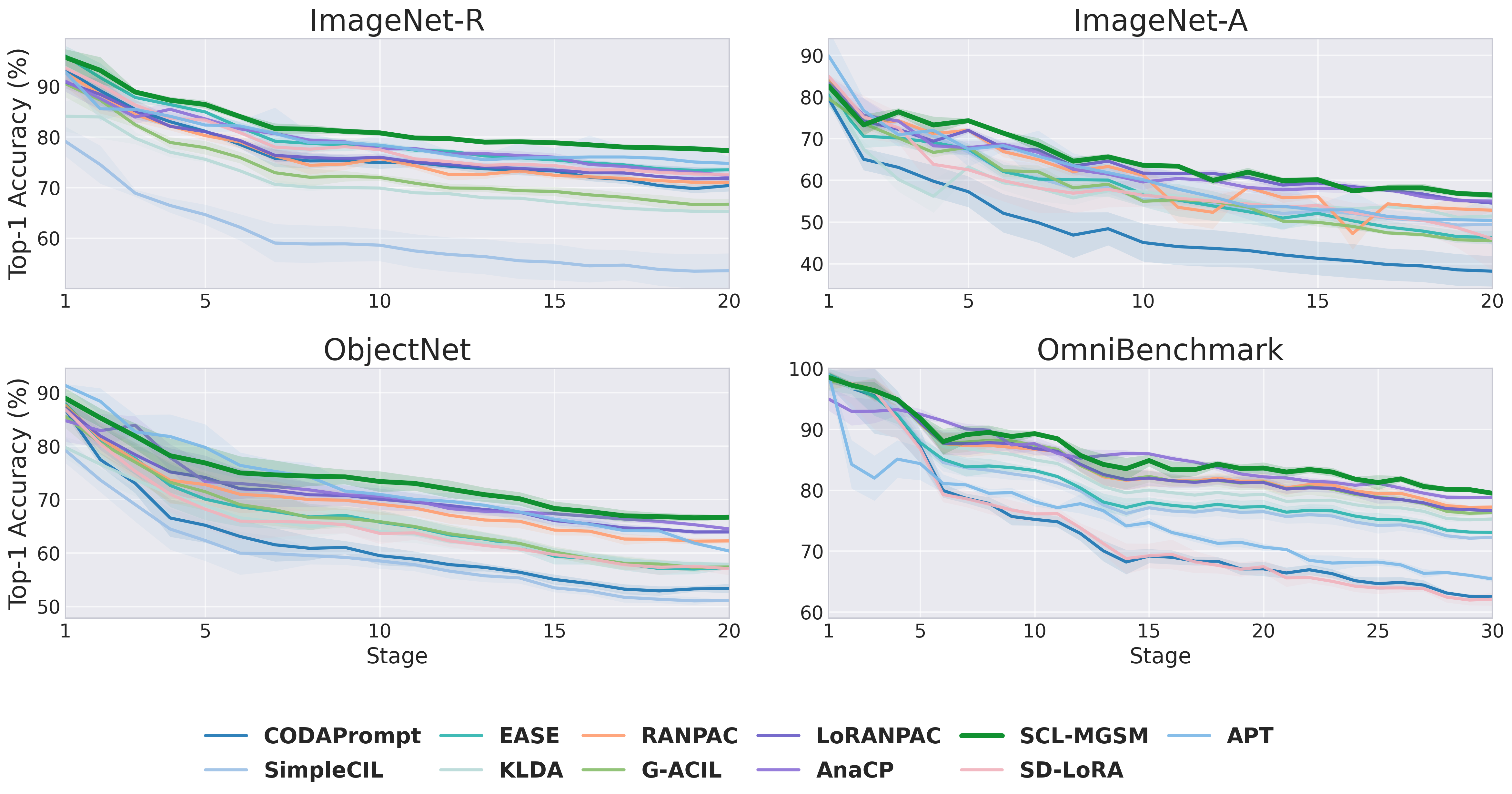}
    \caption{\textbf{Top-1 accuracy during sequential continual learning using ViT-B/16-IN21K.}}
    \label{fig:seed2025_curves}
\end{figure*}

\subsection[Eigenvalue Analysis of the Recursive Matrix Pt]{Eigenvalue Analysis of the Recursive Matrix $\boldsymbol{P}_t$}
\label{appendix_eigs}
As shown in Eq.~(\ref{eq:P_update_simple}), $\boldsymbol{P}_t$ accumulates the random-feature Gram matrices across stages. Its conditioning directly governs the numerical stability of the recursive ridge updates.
To further investigate, Figure~\ref{fig:eigs_cond} plots the maximum eigenvalue, minimum eigenvalue, and condition number of $\boldsymbol{P}_t$ on ImageNet-A across 20 incremental stages.
Under RI, the maximum eigenvalue increases rapidly while the minimum eigenvalue decays, causing the condition number to grow by roughly an order of magnitude.
Under MGSM, the maximum eigenvalue grows more slowly and the minimum eigenvalue remains stable, keeping the condition number consistently lower than that of RI.
These results further suggest that MGSM supports more stable continual learning by maintaining a better-conditioned $\boldsymbol{P}_t$, rather than depending primarily on stronger ridge regularization.

\begin{figure}[t!]
    \centering
    \includegraphics[width=0.98\linewidth]{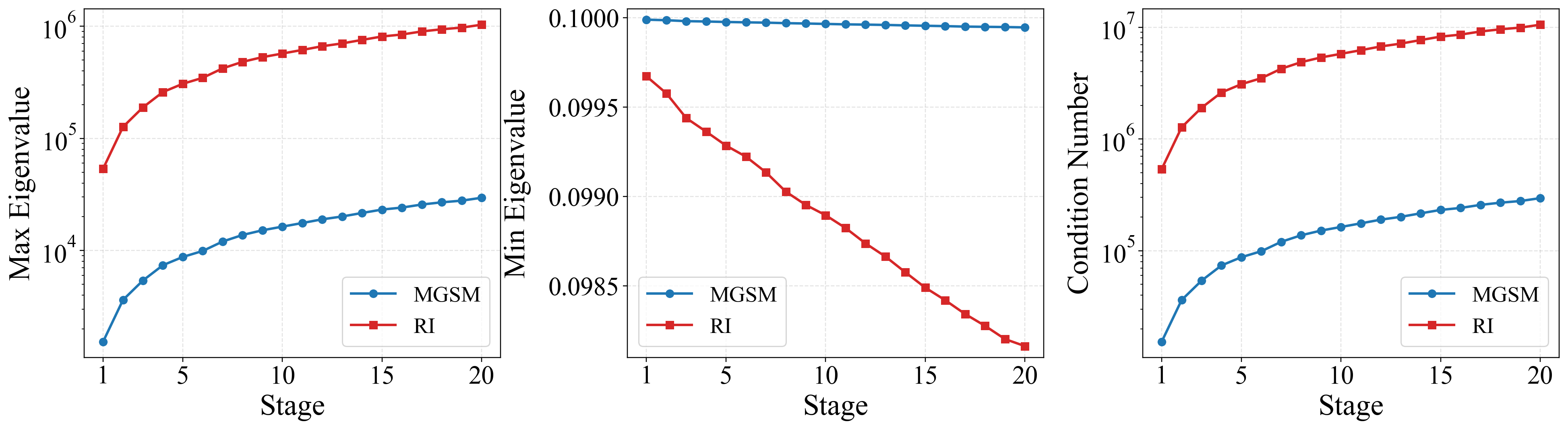}
    \caption{\textbf{Eigenvalue and condition number analysis of $\boldsymbol{P}_t$.} Left: maximum eigenvalue. Center: minimum eigenvalue. Right: condition number of $\boldsymbol{P}_t$.}
    \label{fig:eigs_cond}
\end{figure}

\section{Limitation and Future Work}

SCL-MGSM leverages initial-task data to guide RPL construction via the MGSM criterion, and the resulting RPL is then frozen throughout all subsequent continual learning stages. This is the first paradigm in RPL-based analytic continual learning that uses initial-task supervision to determine which random bases to retain, and it is well suited to scenarios where successive tasks share a certain degree of statistical similarity with the initial task. In practice, this assumption is commonly satisfied: image recognition tasks generally share low-level visual patterns such as edges, textures, and shapes, so even tasks from different domains exhibit non-trivial statistical overlap in the PTM feature space. Our experiments across seven benchmarks with varying domain gaps support this observation. Moreover, as PTMs are trained on increasingly large and diverse datasets, the learned representations become more general, further strengthening this cross-task similarity.

However, if the domain gap between the initial task and later tasks is extreme, the RPL constructed from initial-task statistics may not adequately span the feature directions required by subsequent tasks, potentially limiting representational capacity. In future work, we plan to extend SCL-MGSM toward a dynamic RPL paradigm that can progressively adapt the random projection space across continual learning stages, allowing the RPL to incorporate new task-relevant directions as they emerge while preserving the stability of previously learned representations.

%% file: main.bbl
\begin{thebibliography}{10}
\providecommand{\url}[1]{\texttt{#1}}
\providecommand{\urlprefix}{URL }
\providecommand{\doi}[1]{https://doi.org/#1}

\bibitem{barbu2019objectnet}
Barbu, A., Mayo, D., Alverio, J., Luo, W., Wang, C., Gutfreund, D., Tenenbaum, J., Katz, B.: Objectnet: A large-scale bias-controlled dataset for pushing the limits of object recognition models. NeurIPS  \textbf{32} (2019)

\bibitem{blum2012density}
Blum, K.: Density matrix theory and applications, vol.~64. Springer Science \& Business Media (2012)

\bibitem{chen2025achieving}
Chen, H., Wang, P., Zhou, Z., Zhang, X., Wu, Z., Jiang, Y.G.: Achieving more with less: Additive prompt tuning for rehearsal-free class-incremental learning. In: Proceedings of the IEEE/CVF International Conference on Computer Vision. pp. 340--349 (2025)

\bibitem{chen2022adaptformer}
Chen, S., Ge, C., Tong, Z., Wang, J., Song, Y., Wang, J., Luo, P.: Adaptformer: Adapting vision transformers for scalable visual recognition. Advances in Neural Information Processing Systems  \textbf{35},  16664--16678 (2022)

\bibitem{cover1965geometrical}
Cover, T.M.: Geometrical and statistical properties of systems of linear inequalities with applications in pattern recognition. IEEE transactions on electronic computers (3),  326--334 (1965)

\bibitem{french1999catastrophic}
French, R.M.: Catastrophic forgetting in connectionist networks. Trends in cognitive sciences  \textbf{3}(4),  128--135 (1999)

\bibitem{greville1960some}
Greville, T.: Some applications of the pseudoinverse of a matrix. SIAM review  \textbf{2}(1),  15--22 (1960)

\bibitem{hendrycks2021many}
Hendrycks, D., Basart, S., Mu, N., Kadavath, S., Wang, F., Dorundo, E., Desai, R., Zhu, T., Parajuli, S., Guo, M., et~al.: The many faces of robustness: A critical analysis of out-of-distribution generalization. In: Proceedings of the IEEE/CVF international conference on computer vision. pp. 8340--8349 (2021)

\bibitem{hendrycks2021natural}
Hendrycks, D., Zhao, K., Basart, S., Steinhardt, J., Song, D.: Natural adversarial examples. In: CVPR. pp. 15262--15271 (2021)

\bibitem{jia2022visual}
Jia, M., Tang, L., Chen, B.C., Cardie, C., Belongie, S., Hariharan, B., Lim, S.N.: Visual prompt tuning. In: European conference on computer vision. pp. 709--727. Springer (2022)

\bibitem{kirkpatrick2017overcoming}
Kirkpatrick, J., Pascanu, R., Rabinowitz, N., Veness, J., Desjardins, G., Rusu, A.A., Milan, K., Quan, J., Ramalho, T., Grabska-Barwinska, A., et~al.: Overcoming catastrophic forgetting in neural networks. Proceedings of the national academy of sciences  \textbf{114}(13),  3521--3526 (2017)

\bibitem{li2024harnessing}
Li, D., Wang, T., Chen, J., Dai, W., Zeng, Z.: Harnessing neural unit dynamics for effective and scalable class-incremental learning. arXiv preprint arXiv:2406.02428  (2024)

\bibitem{li2023crnet}
Li, D., Zeng, Z.: Crnet: A fast continual learning framework with random theory. IEEE Transactions on Pattern Analysis and Machine Intelligence  \textbf{45}(9),  10731--10744 (2023)

\bibitem{lian2022scaling}
Lian, D., Zhou, D., Feng, J., Wang, X.: Scaling \& shifting your features: A new baseline for efficient model tuning. Advances in Neural Information Processing Systems  \textbf{35},  109--123 (2022)

\bibitem{mccloskey1989catastrophic}
McCloskey, M., Cohen, N.J.: Catastrophic interference in connectionist networks: The sequential learning problem. In: Psychology of learning and motivation, vol.~24, pp. 109--165. Elsevier (1989)

\bibitem{mcdonnell2024ranpac}
McDonnell, M.D., Gong, D., Parvaneh, A., Abbasnejad, E., van~den Hengel, A.: Ranpac: Random projections and pre-trained models for continual learning. Advances in Neural Information Processing Systems  \textbf{36} (2024)

\bibitem{meng2025diffclass}
Meng, Z., Zhang, J., Yang, C., Zhan, Z., Zhao, P., Wang, Y.: Diffclass: Diffusion-based class incremental learning. In: European Conference on Computer Vision. pp. 142--159. Springer (2025)

\bibitem{momeni2025continual}
Momeni, S., Mazumder, S., Liu, B.: Continual learning using a kernel-based method over foundation models. In: Proceedings of the AAAI Conference on Artificial Intelligence. vol.~39, pp. 19528--19536 (2025)

\bibitem{momenianacp}
Momeni, S., Xiao, C., Liu, B.: Anacp: Toward upper-bound continual learning via analytic contrastive projection. In: The Thirty-ninth Annual Conference on Neural Information Processing Systems

\bibitem{oquab2023dinov2}
Oquab, M., Darcet, T., Moutakanni, T., Vo, H., Szafraniec, M., Khalidov, V., Fernandez, P., Haziza, D., Massa, F., El-Nouby, A., et~al.: Dinov2: Learning robust visual features without supervision. arXiv preprint arXiv:2304.07193  (2023)

\bibitem{panos2023first}
Panos, A., Kobe, Y., Reino, D.O., Aljundi, R., Turner, R.E.: First session adaptation: A strong replay-free baseline for class-incremental learning. In: Proceedings of the IEEE/CVF International Conference on Computer Vision. pp. 18820--18830 (2023)

\bibitem{peng2024loranpac}
Peng, L., Elenter, J., Agterberg, J., Ribeiro, A., Vidal, R.: Loranpac: Low-rank random features and pre-trained models for bridging theory and practice in continual learning. arXiv preprint arXiv:2410.00645  (2024)

\bibitem{peng2023ideal}
Peng, L., Giampouras, P., Vidal, R.: The ideal continual learner: An agent that never forgets. In: International Conference on Machine Learning. pp. 27585--27610. PMLR (2023)

\bibitem{qiao2023prompt}
Qiao, J., Tan, X., Chen, C., Qu, Y., Peng, Y., Xie, Y., et~al.: Prompt gradient projection for continual learning. In: The Twelfth International Conference on Learning Representations (2023)

\bibitem{smith2023coda}
Smith, J.S., Karlinsky, L., Gutta, V., Cascante-Bonilla, P., Kim, D., Arbelle, A., Panda, R., Feris, R., Kira, Z.: Coda-prompt: Continual decomposed attention-based prompting for rehearsal-free continual learning. In: Proceedings of the IEEE/CVF Conference on Computer Vision and Pattern Recognition. pp. 11909--11919 (2023)

\bibitem{sun2024mos}
Sun, H.L., Zhou, D.W., Zhao, H., Gan, L., Zhan, D.C., Ye, H.J.: Mos: Model surgery for pre-trained model-based class-incremental learning. arXiv preprint arXiv:2412.09441  (2024)

\bibitem{tylavsky2005generalization}
Tylavsky, D.J., Sohie, G.R.: Generalization of the matrix inversion lemma. Proceedings of the IEEE  \textbf{74}(7),  1050--1052 (2005)

\bibitem{wang2017stochastic}
Wang, D., Li, M.: Stochastic configuration networks: Fundamentals and algorithms. IEEE transactions on cybernetics  \textbf{47}(10),  3466--3479 (2017)

\bibitem{wang2022learning}
Wang, Z., Zhang, Z., Lee, C.Y., Zhang, H., Sun, R., Ren, X., Su, G., Perot, V., Dy, J., Pfister, T.: Learning to prompt for continual learning. In: Proceedings of the IEEE/CVF conference on computer vision and pattern recognition. pp. 139--149 (2022)

\bibitem{rw2019timm}
Wightman, R.: Pytorch image models. \url{https://github.com/rwightman/pytorch-image-models} (2019). \doi{10.5281/zenodo.4414861}

\bibitem{wu2025sd}
Wu, Y., Piao, H., Huang, L.K., Wang, R., Li, W., Pfister, H., Meng, D., Ma, K., Wei, Y.: Sd-lora: Scalable decoupled low-rank adaptation for class incremental learning. arXiv preprint arXiv:2501.13198  (2025)

\bibitem{yan2021dynamically}
Yan, S., Xie, J., He, X.: Der: Dynamically expandable representation for class incremental learning. In: Proceedings of the IEEE/CVF conference on computer vision and pattern recognition. pp. 3014--3023 (2021)

\bibitem{yu2024boosting}
Yu, J., Zhuge, Y., Zhang, L., Hu, P., Wang, D., Lu, H., He, Y.: Boosting continual learning of vision-language models via mixture-of-experts adapters. In: Proceedings of the IEEE/CVF Conference on Computer Vision and Pattern Recognition. pp. 23219--23230 (2024)

\bibitem{zhang2022benchmarking}
Zhang, Y., Yin, Z., Shao, J., Liu, Z.: Benchmarking omni-vision representation through the lens of visual realms. In: ECCV. pp. 594--611. Springer (2022)

\bibitem{zhou2024revisiting}
Zhou, D.W., Cai, Z.W., Ye, H.J., Zhan, D.C., Liu, Z.: Revisiting class-incremental learning with pre-trained models: Generalizability and adaptivity are all you need. International Journal of Computer Vision pp. 1--21 (2024)

\bibitem{zhou2024expandable}
Zhou, D.W., Sun, H.L., Ye, H.J., Zhan, D.C.: Expandable subspace ensemble for pre-trained model-based class-incremental learning. In: Proceedings of the IEEE/CVF Conference on Computer Vision and Pattern Recognition. pp. 23554--23564 (2024)

\bibitem{zhu2021prototype}
Zhu, F., Zhang, X.Y., Wang, C., Yin, F., Liu, C.L.: Prototype augmentation and self-supervision for incremental learning. In: Proceedings of the IEEE/CVF Conference on Computer Vision and Pattern Recognition. pp. 5871--5880 (2021)

\bibitem{zhuang2024g}
Zhuang, H., Chen, Y., Fang, D., He, R., Tong, K., Wei, H., Zeng, Z., Chen, C.: G-acil: Analytic learning for exemplar-free generalized class incremental learning. arXiv preprint arXiv:2403.15706  (2024)

\bibitem{zhuang2022acil}
Zhuang, H., Weng, Z., Wei, H., Xie, R., Toh, K.A., Lin, Z.: Acil: Analytic class-incremental learning with absolute memorization and privacy protection. Advances in Neural Information Processing Systems  \textbf{35},  11602--11614 (2022)

\bibitem{zou2025structural}
Zou, H., Zang, Y., Ji, X.: Structural features of the fly olfactory circuit mitigate the stability-plasticity dilemma in continual learning. arXiv preprint arXiv:2502.01427  (2025)

\bibitem{zou2025fly}
Zou, H., Zang, Y., Xu, W., Ji, X.: Fly-cl: A fly-inspired framework for enhancing efficient decorrelation and reduced training time in pre-trained model-based continual representation learning. arXiv preprint arXiv:2510.16877  (2025)

\bibitem{zou2025flylora}
Zou, H., Zang, Y., Xu, W., Zhu, Y., Ji, X.: Flylo{RA}: Boosting task decoupling and parameter efficiency via implicit rank-wise mixture-of-experts. In: The Thirty-ninth Annual Conference on Neural Information Processing Systems (2025)

\end{thebibliography}
